\RequirePackage{fix-cm}
\documentclass{svjour3}                     
\smartqed  

\usepackage{amsmath,amsfonts,amssymb,cases,mathdots,color,colordvi,appendix,url,graphicx,multirow,caption,subcaption,float,lscape}
\usepackage{geometry}
\usepackage{hyperref}
\usepackage{mathptmx}
\usepackage[justification=centering]{caption}
\newtheorem{assumption}{Assumption}

\begin{document}

\title{Matrix optimization based Euclidean embedding with outliers}
\author{Qian Zhang         \and
        Xinyuan Zhao \and
        Chao Ding
}
\institute{Q. Zhang \at
              College of Applied Sciences, Beijing University of Technology, Beijing, P.R. China. \\
              \email{zhangqian@emails.bjut.edu.cn}           
                         \and
           X. Y. Zhao \at
              College of Applied Sciences, Beijing University of Technology, Beijing, P.R. China. The research of this author was supported by the National Natural Science Foundation of China under projects No. 11871002 and the General Program of Science and Technology of Beijing Municipal Education Commission No. KM201810005004.\\ 
              \email{xyzhao@bjut.edu.cn}
              \and
              C. Ding \at
              Institute of Applied Mathematics, Academy of Mathematics and Systems Science, Chinese Academy of Sciences, Beijing,  P.R. China. The research of this author was supported by the National Natural Science Foundation of China under projects
              No. 12071464, No. 11671387,  No. 11531014 and No. 11688101 and the Beijing Natural Science Foundation (Z190002).\\ 
              \email{dingchao@amss.ac.cn}
}

\date{This version: December 18, 2020}

\maketitle

\begin{abstract}
Euclidean embedding from noisy observations containing outlier errors is an important and challenging problem in statistics and machine learning. Many existing methods would struggle with outliers due to a lack of detection ability. In this paper, we propose a matrix optimization based embedding model that can produce reliable embeddings and identify the outliers jointly. We show that the estimators obtained by the proposed method satisfy a non-asymptotic risk bound, implying that the model provides a high accuracy estimator with high probability when the order of the sample size is roughly the degree of freedom up to a logarithmic factor. Moreover, we show that under some mild conditions, the proposed model also can identify the outliers without any prior information with high probability. Finally, numerical experiments demonstrate that the matrix optimization-based model can produce configurations of high quality and successfully identify outliers even for large networks.  

\end{abstract}

\section{Introduction}\label{section:introduction}

Finding a complete set in a low-dimensional Euclidean space  from partial noisy Euclidean distance observations, so-called embedding, is an important distance geometric problem in data science applications. In particular, when the distances are assumed to be measured in higher dimensional spaces, this leads to a typical nonlinear dimensional reduction which is widely used in statistics and machine learning.  One of the biggest challenges in embedding is that some noisy distance observations are usually contaminated with positive outlier errors. Due to the nature of applications, the outlier errors are usually much larger than the commonly assumed zero-mean measurement noises. Moreover, in many applications, the outlier errors are even larger than the true distance (e.g., the Non-Line-Of-Sight (NLOS) errors from wireless sensor network localization \cite{PNNeiyer2005,stoica2006lecture,GChong09}), which is the main reason for the distance based methods often fail in practical embedding applications. 

Needless to say,  it is important to mitigate outlier errors from the observation distances for embedding. In literature, for some  applications (e.g., wireless communication), one way to mitigate outlier propagation is to develop some methods to identify outlier errors through prior information, such as the outlier distribution and physical characteristics of networks \cite{VBurhrer07,RUrruela04}. For the overview of various outlier identification techniques and optimization methods in wireless communication applications, see the nice survey \cite{GChong09}. However, in most applications, the prior information of outliers is either technically non-available or costly to obtain due to hardware limitations. Therefore, it is even more crucial to identify and mitigate the outlier propagation from the observed distance data without prior information. To this end, different matrix optimization models are proposed and become popular in applications.  For instance, in wireless sensor localization, different semidefinite programming (SDP) based methods are proposed by \cite{CWWSPoor12,VSBuehrer13,YCChampagne14}, and numerical experiments demonstrate that SDP based models can provide descent estimations even without prior information on outliers for some small-scale applications. More recently, based on the concept of the Euclidean distance matrix (EDM), a new matrix optimization model for the outlier mitigation has been proposed in \cite{DQi2017b} (see Section \ref{section:Model} for details). Numerical tests on both  simulations and real-world applications show that the EDM based method  proposed in \cite{DQi2017b} can produce high quality embeddings without prior information even for large-scale networks. Numerical experiments show that one of the main advantages of the EDM model \cite{DQi2017b}  comparing with the existing SDP approaches is that the EDM model usually is able to identify outlier errors index sets (i.e., the index sets of observation distances which contain  outlier errors).   However, there is no theoretical guarantee on the outlier detection ability  provided in \cite{DQi2017b}. The main purpose of this paper is to  study the statistical performance analysis of the EDM based embedding with outliers by establishing the recovery error bounds and the embedding dimension and outlier detection guarantee. 

In general, the EDM based embedding model proposed in \cite{DQi2017b} belongs the category of low-rank matrix approximation problems \cite{Mesbahi98,Fazel02}, which had many exciting developments recently and attracted much attention from optimization and machine learning communities. More precisely, in principle, the proposed EDM model is in line with the general framework of robust principal component analysis (Robust PCA) \cite{CSPWillsky11,CLMWright11}, i.e., estimating an unknown low-rank matrix $\overline{X}\in \mathbb{R}^{m\times n}$ from a collection of partially noisy observed elements $\widetilde{X}_{ij}=\overline{X}_{ij}+\overline{S}_{ij}+\daleth_{ij}$, $(i,j)\in \Omega\subseteq \{1,\ldots,m\}\times \{1,\ldots,n\}$, where $\overline{S}$ is a sparse matrix consisting of outliers, $\daleth$ represents the random noise and $\Omega$ is the observation index subset. 

Enlightened by the previous tremendous success of the convex matrix optimization approaches in low-rank matrix completion \cite{CR09,CPlan10},  Chandrasekaran et al. \cite{CSPWillsky11} first study the Robust PCA for the case that the elements $\widetilde{X}$ are fully observed and  without noise. In particular,  Chandrasekaran et al. \cite{CSPWillsky11} showed that under the rank-sparsity incoherent property, the unknown true low-rank matrix $\overline{X}$ and sparse outlier matrix $\overline{S}$ can be recovered exactly based on the convex ``nuclear norm plus $l_1$-norm" approach. For the general setting with missing observations, by employing the previous developed  probability analysis techniques for exact matrix completion problems \cite{CR09,Recht11,Gross11}, Cand\`{e}s et al. \cite{CLMWright11} provided the probabilistic guarantees for exact recovery of the convex ``nuclear norm plus $l_1$-norm" approach for Robust PCA. Later a sharper probabilistic exact recovery guarantee was established by Chen et al. \cite{CJSCaramanis13} focused on high-dimensional statistical settings (i.e., the sample size is smaller than $mn$).    

For the more realistic noisy setting, Zhou et al. \cite{ZWCMa10} proposed a convex ``nuclear norm plus $l_1$-norm"  constrained matrix optimization model and studied its statistical performance guarantees, and later the nuclear norm plus $l_1$-norm penalized formulation was studied by Hsu et al. \cite{HKZhang11}. Based on the unified restricted strong convexity (RSC) framework introduced by \cite{NRWYu12,NWainwright12},  Agarwal et al. \cite{ANW12} obtain a sharper statistical error bound for the nuclear norm plus $l_1$-norm penalized model. However, in prior studies \cite{ZWCMa10,HKZhang11,ANW12} on the noisy robust PCA, the performance guarantee results are all based on the full observation assumption, which may not be practical in applications. In \cite{WLee17}, Wong and Lee established an estimation error bound for the noisy robust PCA, with an assumption that the number of observed entries is in the order of $mn$. However, this also may not be useful for high-dimensional applications, since the sample sizes there are usually much smaller. For high-dimensional settings, under the boundedness assumptions on the true  low-rank matrix $\overline{X}$ and sparse outlier matrix $\overline{S}$,  Klopp, et al. \cite{KLTsybakov17} derived a statistical estimation error bounds for the $l_\infty$ constrained convex program based estimators, in which the minimal sample size required for a faithful estimation is roughly on the order of $\max\{m,n\}r\log(n)$. Recently, Chen, et al. \cite{CFMYan2020}, improved and derived a near-optimal statistical guarantee of the convex nuclear norm plus $l_1$-norm penalized model for the Robust PCA by building up the connection between the convex estimations and an auxiliary nonconvex optimization algorithm.

The results mentioned above are all about the ``classical" robust PCA, in the sense that  for the proposed models there is no ``hard-constraints", e.g., the noisy correlation matrix recovery (i.e., a positive semidefinite matrix whose diagonal elements are all ones) and the EDM estimation considered in this paper. However, these ``hard-constraints"  are usually crucial and must be satisfied in the convex estimation models in many applications e.g., the EDM embedding. Consequently, the results obtained in \cite{WLee17,KLTsybakov17,CFMYan2020} have become inadequate in these applications. For the correlation matrix estimation problem, Wu \cite{Wu2014} first studied the probabilistic guarantees of the Robust PCA with ``hard-constraints"  for both noiseless and noisy cases. The main techniques employed in \cite{Wu2014} are the unified restricted strong convexity (RSC) framework introduced by \cite{NRWYu12,NWainwright12} and a matrix Bernstein inequality (cf. e.g., \cite{Vershynin10}), which are wildly used in the study of statistical performance guarantees of convex models in matrix completion problem (e.g., \cite{Gross11,Recht11,KLTsybakov17,NWainwright12,Klopp14,MPSun16,Miao2013}) and the EDM embedding problem without outliers \cite{DQi2017}. 

In order to establish the theoretical performance analysis of the convex matrix optimization model for EDM embedding with outliers, we first adopt the  error bound analysis approach introduced in \cite{Wu2014} to EDM embedding with outliers and obtain the statistical guarantee of the convex estimation model. Furthermore, based on the resulting error bound results, we show that under some wild conditions, with high probability, the convex EDM estimator will recover the true unknown embedding dimension. Simultaneously, we also show that the outlier estimator obtained by the convex matrix optimization model will recover the index set of the support set of the unknown outliers with the same probability. Finally, we verify the proposed theoretical results by numerical experiments.  


%
%
%
%


The remaining parts of this paper are organized as follows.  We briefly introduce the matrix optimization based EDM embedding with outliers model originally proposed in \cite{DQi2017b}. Section \ref{section:error-bounds} contains the statistical recovery error bounds for the EDM embedding EDM model.  In Section \ref{section:recovery-dimension-outlier}, we establish the probability recovery guarantee of the embedding dimensionality and outlier detection. We verify the theoretical results obtained in Sections \ref{section:error-bounds} and \ref{section:recovery-dimension-outlier} through numerical examples in Section \ref{section:numerical}. We conclude the paper in Section \ref{section:conclusion}.

Below are some common notations to be used in this paper:
\begin{itemize}
	\item For any $Z\in\mathbb{R}^{m\times n}$, we denote by $Z_{ij}$ the $(i,j)$-th entry of $Z$.
%
	\item We use $``\circ"$ to denote the Hadamard product between matrices, i.e., for any two matrices $X$ and $Y$ in $\mathbb{R}^{m\times n}$ the $(i,j)$-th entry of $  Z:= X\circ Y \in \mathbb{R}^{m\times n}$ is
	$Z_{ij}=X_{ij} Y_{ij}$.
	\item For any $Z\in\mathbb{R}^{m\times n}$, we use $Z^{1/2}\in\mathbb{R}^{m\times n}$ to denote the $m\times n$ matrix whose $(i,j)$-th entry is $Z_{ij}^{1/2}$.
	\item Let ${\bf 1}\in\mathbb{R}^n$ be the vector whose elements are all ones. Denote the $n\times n$ identity matrix by $I$ and the centering matrix by 
	\begin{equation}\label{eq:def-J}
		J := I - {\bf 1}{\bf 1}^T/n.
	\end{equation}
\item  For a given $Z\in\mathbb{S}^n$, we use $\lambda_1(Z)\ge\ldots\ge \lambda_n(Z)$ to denote the eigenvalues of $Z$ (all real and counting multiplicity) arranging in non-increasing order and use $\lambda(Z)$ to denote the vector of the ordered eigenvalues of $Z$. Let $\mathbb{O}^n$ be the set of all $n\times n$ orthogonal matrices.
\item Let ${\rm sgn}(\cdot):\mathbb{R}^{m\times n}\to \mathbb{R}^{m\times n}$ be the (component-wise) sign function, i.e., for any $Z\in\mathbb{R}^{m\times n}$, ${\rm sgn}(Z)_{ij}=1$ if $Z_{ij}>0$ and ${\rm sgn}(Z)_{ij}=0$ otherwise.
\item For any $z\in \mathbb{R}^n$, we use ${\rm Diag}(z)$ to denote an $n\times n$ diagonal matrix with $z$ on the main diagonal. Meanwhile, for any $Z\in\mathbb{R}^{n\times n}$, ${\rm diag}(Z)$ denotes the column vector consisting of  all the diagonal entries of $Z$ being arranged from the first to the last.
\end{itemize}


\section{The EDM based embedding with outliers}\label{section:Model}

Let $p_i$, $i=1,\ldots,n$ be $n$ points in a $r$-dimensional subspace. For each $i,j\in\{1,\ldots,n\}$, denote $\bar{d}_{ij}\ge 0$ the  distance between the $i$-th and $j$-th points on the $r$-dimensional subspace. A $n\times n$ matrix $\overline{D}$ is called Euclidean distance matrix (EDM) if
$\overline{D}_{ij}=(\bar{d}_{ij})^2$ for $i,j=1,\ldots,n$. An alternative definition of EDM that does not involve any embedding points $\{p_i\}$ can be described as follows. Let $\mathbb{S}_h^n$ be the hollow subspace of $\mathbb{S}^n$, 
i.e., $\mathbb{S}_h^n:=\left\{X\in\mathbb{S}^n\mid {\rm diag}(X)=0 \right\}$.
Define the almost positive semidefinite cone $\mathbb{K}^n_+$ by
\begin{equation}\label{eq:def_alPSD}
	\mathbb{K}^n_+:=\left\{A\in\mathbb{S}^n \mid x^T A x \ge 0 \ \forall\, x \in {\bf 1}^{\perp}\right\} 
	= \left\{A\in\mathbb{S}^n \mid JAJ \succeq 0 \right\},
\end{equation}
where ${\bf 1}^{\perp}:=\{x\in\mathbb{R}^n\mid{\bf 1}^T x=0\}$.
It is well-known \cite{schoenberg1935remarks,young1938discussion} that $D\in\mathbb{S}^n$ is EDM if and only if $-D\in \mathbb{S}_h^n\cap\mathbb{K}^n_+$. Moreover, the embedding dimension is determined by the rank of the doubly centered matrix $J\overline{D}J$, i.e., $r = {\rm rank}(J\overline{D}J)$. Given a true EDM $\overline{D}$, since $-J\overline{D}J$ is positive semidefinite, its spectral decomposition can be written as
\[
- \frac {1}{2} J\overline{D}J = P {\rm Diag}(\lambda_1, \ldots, \lambda_n) P^T,
\]
where $P\in\mathbb{O}^n$ and
$\lambda_1 \ge \lambda_2 \ge \cdots \ge \lambda_n \ge 0$ are the eigenvalues in  nonincreasing order. Let $P_1$ be the submatrix consisting of the first $r$ columns (eigenvectors) in $P$.
One set of the embedding points are
\begin{equation} \label{Eq:cMDS}
	\left( \begin{array}{l}
		p_1^T \\ \vdots \\
		p_n^T
	\end{array} \right) = P_1 {\rm Diag}(\sqrt{\lambda_1}, \ldots, \sqrt{\lambda_r}).
\end{equation}
In order to find a set of relative embedding points $\{p_i\}$, we are interesting in estimating the ture EDM $\overline{D}$ from the partial noisy observation  distances $\tilde{d}_{ij}$. The basic noisy model takes the following form
\begin{equation}\label{eq:distance_estimation_o}
	\tilde{d}_{ij}=\bar{d}_{ij} +\eta\xi_{ij}, \quad i,j\in\{1,\ldots,n\}. 
\end{equation}
where $\xi_{ij}$ are i.i.d. noise errors with $\mathbb{E}(\xi_{ij})=0$ and $\mathbb{E}(\xi_{ij}^2)=1$, $\eta>0$ is a noise magnitude control factor.

Unlike the standard zero-mean noise assumption, we are attractive to the case where the distance measurements $\tilde{d}_{ij}$ are also contaminated with the errors arising from outliers, which usually have significant positive biases and cause the measured distances $\tilde{d}_{ij}$ significantly diverging from actual values $\bar{d}_{ij}$.  The errors from outliers frequently appear in many applications such as the Non-Line-Of-Sight (NLOS) errors from wireless sensor network localization \cite{PNNeiyer2005,stoica2006lecture,GChong09}, the errors arising from outliers in manifold learning and others \cite{hodge2004survey,kuhn2013applied}.

We use $\bar{s}\in\mathbb{S}^n$ to represent the outlier errors, whose elements are either zero or  positive accordingly. Therefore, the basic noisy model \eqref{eq:distance_estimation_o} then takes the following form 
\begin{equation}\label{eq:distance_estimation}
	0\le \tilde{d}_{ij}=\bar{d}_{ij}+\bar{s}_{ij} +\eta\xi_{ij} , \quad i,j\in\{1,\ldots,n\}. 
\end{equation}
For notational simplicity, we define the {\it outlier matrix} $\overline{S}\in\mathbb{S}^n$ as follows 
\begin{equation}\label{eq:def-outlier matrix}
	\overline{S}_{ij}:=\bar{s}_{ij}(\bar{s}_{ij}+2\bar{d}_{ij})\ge 0,\quad i,j\in\{1,\ldots,n\}.
\end{equation} 
Then, by \eqref{eq:distance_estimation}, we know that 
\begin{equation}\label{eq:def-D-tilde}
	\widetilde{D}_{ij}=\tilde{d}_{ij}^2=(\bar{d}_{ij}  +\bar{s}_{ij}+\eta\xi_{ij})^2= \overline{D}_{ij}+\overline{S}_{ij} + \daleth_{ij}, \quad i,j\in\{1,\ldots,n\},
\end{equation}
where $\widetilde{D}$ and $\overline{D}\in\mathbb{S}^n$ are the observation and unknown true EDMs, whose $(i,j)$-element are $\tilde{d}_{ij}^2$ and $\bar{d}_{ij}^2$, respectively, and $\daleth\in\mathbb{S}^n$ is defined by 
\begin{equation}\label{eq:def-daleth}
	\daleth_{ij}:=2(\bar{d}_{ij}+\bar{s}_{ij})\eta\xi_{ij}+\eta^2\xi_{ij}^2, \quad i,j\in\{1,\ldots,n\}.
\end{equation}
Without loss of generality, we always assume that  the magnitude of measurement error $|\eta\xi_{ij}|$ is strictly  smaller than the true distance $d_{ij}$ for each  $i,j\in\{1,\ldots,n\}$. Thus, since $\tilde{d}_{ij}\ge 0$, we know that for any  $i,j\in\{1,\ldots,n\}$,
\[
\bar{s}_{ij}=0 \quad \Longleftrightarrow \quad  \overline{S}_{ij}=0.
\]

For each $i\in\{1,\ldots,n\}$, we use ${\bf e}_i\in\mathbb{R}^n$ to denote the $i$-th canonical basis of $\mathbb{R}^n$. Let $\mathbb{S}^n_h\subseteq \mathbb{S}^n$ be the hollow space, i.e.,
$$
\mathbb{S}^n_h:=\left\{ Z\in\mathbb{S}^n\mid Z_{ii}=0,\ i=1,\ldots,n\right\},
$$ 
whose dimension equals to $d_{\mathbb{S}^n_h}\equiv n(n-1)/2$. Let $\left\{ \frac{1}{{2}} ({\bf e}_i {\bf e}_j^T + {\bf e}_j{\bf e}_i^T) \right\}_{1\le i<j\le n}$ be the standard basis matrices of $\mathbb{S}^n_h$. 

For the given observation set $\Omega$, let $\{X_1, \ldots, X_m\}$ with $m:=|\Omega|$ be the numbered  sampled basis matrices from the standard basis matrices set $\left\{ \frac{1}{{2}} ({\bf e}_i {\bf e}_j^T + {\bf e}_j{\bf e}_i^T) \right\}_{1\le i<j\le n}$. Therefore, the corresponding observation operator ${\cal O}_{\Omega}:\mathbb{S}^n\to\mathbb{R}^{m}$ can be written as
\begin{equation}\label{eq:def_obser_op}
	{\cal O}_{\Omega}(A):=\left(\langle X_1,A\rangle,\ldots,\langle X_m,A\rangle\right)^T\in\mathbb{R}^m, \quad A\in\mathbb{S}^n.
\end{equation}
That is, ${\cal O}_{\Omega}(A)$ samples all the elements $A_{ij}$ specified by $(i,j) \in \Omega$. Let ${\cal O}_{\Omega}^{*}:\mathbb{R}^{m}\to{\mathbb S}^{n}$ be its adjoint, i.e., ${\cal O}_{\Omega}^{*}({\bf z})=\sum_{l=1}^{m}{\bf z}_{l}X_{l},$ ${\bf z}\in\mathbb{R}^{m}$.
Then, we further define the observation vector 
\begin{equation}\label{eq:def-ob-y}
	{\bf y}:={\cal O}_{\Omega}(\widetilde{D}) \in\mathbb{R}^m.
\end{equation} 

Finally, under the assumption that $\overline{S}$ is sparse (i.e., the cardinality of nonzero elements of $\overline{S}$ is small), we may estimate the unknown matrices $\overline{D}$ and $\overline{S}$ by solving the following nonconvex optimization model:
\begin{equation}\label{pr:matrix}
	\begin{array}{cl}
		\min & \displaystyle\frac{1}{2m}\|{\bf y}-\mathcal{O}_\Omega(D+S)\|^2+\rho \|S\|_0\\ [5pt]
		\textrm{s.t.}\; & D\in \mathbb{S}_h^n, \quad -D\in\mathbb{K}_+^n, \quad \textrm{rank}(-JDJ)\le r, \\[5pt]
		& S\in\mathbb{S}^n, \quad S\ge 0. 
	\end{array}
\end{equation}
%
%
%
%
Since the rank constraint and the zero norm $\|\cdot\|_0$ are computational intractable, we may consider the following convex relaxated matrix optimization problem
\begin{equation}\label{pr:cvxopt}
	\begin{array}{cl}
		\min & \Phi_{\rho_D,\rho_S}(D,S):=\displaystyle\frac{1}{2m}\|{\bf y}-\mathcal{O}_\Omega(D+S)\|^2 + \rho_D\langle I-\widetilde{F}, -JDJ \rangle+\rho_{S}\langle E-\widetilde{G}, S\rangle \\ [5pt]
		\textrm{s.t.}\; & D\in \mathbb{S}_h^n, \quad -D\in\mathbb{K}_+^n, \quad  S\ge 0,
	\end{array}
\end{equation}
where $\rho_{D}$ and $\rho_{S}$ are two given positive parameters, $E\in\mathbb{S}^{n}$ is the matrix whose elements are all ones, and $\widetilde{F}$ and $\widetilde{G}$ are two given symmetric matrices. In particular, when both $\widetilde{F}$ and $\widetilde{G}$ vanish, the model \eqref{pr:cvxopt} reduces to the following convex nuclear norm $l_1$-minimization EDM matrix optimization problem:
\begin{equation}\label{pr:NL1EDM}
	\begin{array}{cl}
		\min & \displaystyle\frac{1}{2m}\|{\bf y}-\mathcal{O}_\Omega(D+S)\|^2 + \rho_D\langle I, -JDJ \rangle +\rho_{S}\langle E, S\rangle \\ [5pt]
		\textrm{s.t.}\; & D\in \mathbb{S}_h^n, \quad -D\in\mathbb{K}_+^n, \quad  S\ge 0.
	\end{array}
\end{equation}
We use $(\widehat{D}_m,\widehat{S}_m)$ to denote an optimal solution of the above convex model \eqref{pr:cvxopt}, and in later discussions, we often drop the subscript ``$m$", when the dependence of $(\widehat{D},\widehat{S})$ on the sample size $m$ is clear from the context. Furthermore, we use $(\overline{D},\overline{S})$ to denote the unknown true EDM and outlier matrix.

In this paper, we choose the symmetric matrices $\widetilde{F}\in \mathbb{S}^n$ and $\widetilde{G}\in \mathbb{S}^n$ in the objective function of \eqref{pr:cvxopt} by following the suggestions \cite[(25) and (26)]{MPSun16} (see also \cite[Chapter 5.3]{Wu2014}). Suppose that  the EDM $\widetilde{D}$ and the non-negative matrix $\widetilde{S}$ are the given initial estimators (e.g., the estimators obtained by \eqref{pr:NL1EDM}). Define $\widetilde{F}\in \mathbb{S}^n$ by
\begin{equation}\label{eq:def-F}
	\widetilde{F} =  F(-J\widetilde{D}J),
\end{equation}
where $F:\mathbb{S}^n\to\mathbb{S}^n$ is the spectral operator \cite{DSSToh18,DSSToh20} associated with the symmetric function (cf. \cite[Definition 1]{DSSToh18}) $f:\mathbb{R}^n\to\mathbb{R}^n$ defined by
\begin{equation}\label{eq:def-sym-f}
	f_i(x) = \left\{ \begin{array}{ll}
		\phi\left(x_i/\max_{l}\{x_l\}\right) & \mbox{if $x\neq 0$}, \\ [3pt]
		0 & \mbox{otherwise,}
	\end{array}\right.\quad x\in\mathbb{R}^n
\end{equation}
with the scalar function $\phi:\mathbb{R}\to\mathbb{R}$ takes the form
\begin{equation}\label{eq:def-phi-scalar}
	\phi(t) =  (1+\varepsilon^{\tau})\frac{t^{\tau}}{t^{\tau}+\varepsilon^{\tau}}, \quad t\in\mathbb{R}
\end{equation}
for some $\varepsilon>0$ and $\tau>0$. Meanwhile, we define the symmetric matrix $\widetilde{G}\in \mathbb{S}^n$ with respect to $\widetilde{S}$ by
\begin{equation}\label{eq:def-G}
	\widetilde{G}_{ij} = \left\{ \begin{array}{ll}
		\phi\big(\widetilde{S}_{ij}/\max_{k,l}\{\widetilde{S}_{kl}\}\big) & \mbox{if $\widetilde{S}\neq 0$}, \\ [3pt]
		0 & \mbox{otherwise,}
	\end{array}\right.  \quad (i,j)\in\{1,\ldots,n\},
\end{equation} 
where $\phi:\mathbb{R}\to\mathbb{R}$ is the scalar function defined by \eqref{eq:def-phi-scalar}.

Throughout this paper, the following condition is assumed to hold, which ensures the existent of the optimal solution of \eqref{pr:cvxopt}.

\begin{assumption}\label{ass:level set}
	There exists constant $\overline{\rho}_D>0$ and $\overline{\rho}_S>0$ such that for any $\alpha$, the level set 
	$$
	L_{\overline{\rho}_D,\overline{\rho}_S}(\alpha):= \left\{ (D,S)\in \mathbb{S}^n\mid \Phi_{\overline{\rho}_D,\overline{\rho}_S}(D,S)\le \alpha,\ D\in\mathbb{S}^n_h,\ -D\in \mathbb{K}^n_+,\ S\ge 0\right\}
	$$ of \eqref{pr:cvxopt} is closed and bounded. 
\end{assumption}

It is worth to note that for any $0\le t\le 1$, $\phi(t)\in [0,1]$. Thus, it is easy to check that for any given initial EDM  estimator $\widetilde{D}$ and outlier matrix estimator $\widetilde{S}$, the symmetric matrices $\widetilde{F}$ and $\widetilde{G}$ satisfy $\langle I-\widetilde{F}, -JDJ\rangle \ge 0$ and $\langle E-\widetilde{G},S\rangle \ge 0$. Therefore, we know that for any  $\rho_D\ge \overline{\rho}_D>0$ and $\rho_S\ge \overline{\rho}_S>0$,  $L_{{\rho}_D,{\rho}_S}(\alpha)\subseteq L_{\overline{\rho}_D,\overline{\rho}_S}(\alpha)$ for any $\alpha$. This yields that  under Assumption \ref{ass:level set},  for any  $\rho_D\ge \overline{\rho}_D>0$ and $\rho_S\ge \overline{\rho}_S>0$,  the optimal solution of \eqref{pr:cvxopt} exists. Moreover, under Assumption \ref{ass:level set}, we know that there exist two positive constants $b_D$ and $b_S$ such that the optimal solution  $(\widehat{D},\widehat{S})$ of \eqref{pr:cvxopt}  for any $\rho_D\ge \overline{\rho}_D$ and $\rho_S\ge \overline{\rho}_S$ and the unkown true EDM and outlier matrix $(\overline{D},\overline{S})$ satisfy
\[
\|\widehat{D}\|_{\infty}\le b_D, \quad \|\overline{D}\|_{\infty}\le b_D, \quad \|\widehat{S}\|_{\infty}\le b_S \quad {\rm and} \quad \|\overline{S}\|_{\infty}\le b_S.
\]

\section{Recovery error bounds}\label{section:error-bounds}

In this section, we aim to derive a recovery error bound  for the proposed matrix optimization based EDM embedding model with outliers  \eqref{pr:cvxopt}. Here we adopt the approach introduced by Wu \cite{Wu2014} for studying recovery error bounds of the noisy low-rank and sparse matrix decomposition with fixed basis. Essentially, the proofs are in line with the well-studied unified framework introduced by Negahban et al. \cite{NRWYu12} for high-dimensional analysis of M-estimators with decomposable regularizers, which is used frequently in the study of noisy matrix completion \cite{NWainwright12,Klopp14,Miao2013,MPSun16,Wu2014,DQi2017}.  For the sake of completion, we  include the detailed proofs in Appendix. 

Recall that $\overline{D}$ is the unknown true EDM matrix. Suppose that the positive semidefinite matrix $-J\overline{D}J$ has the eigenvalue decomposition 
\begin{equation}\label{eq:SVD-JbarDJ}
	-J\overline{D}J=\overline{P}\,\overline{\Lambda}\,\overline{P}^T,
\end{equation}  
where $\overline{\Lambda}\in\mathbb{S}^n$ is a diagonal matrix whose diagonal elements are the eigenvalues of $-J\overline{D}J$ arranged in the non-increasing order, and $\overline{P} = \left[\overline{P}_1 \ \  \overline{P}_2\right] \in{\mathbb O}^n$ with $\overline{P}_1 \in {\mathbb R}^{n \times r}$ and $\overline{P}_2\in\mathbb{R}^{n\times (n-r)}$.
We define the generalized geometric center subspace in ${\mathbb S}^n$ with respect to $\overline{P}_1$ by $\mathbb{T} := \left\{ Y \in {\mathbb S}^n \ | \ Y \overline{P}_1 = 0 \right\}$.
Let $\mathbb{T}^{\perp}$ be its orthogonal subspace. Then,
the orthogonal projectors to the two subspaces can hence be calculated respectively by
\begin{equation}\label{def:projection}
	{\cal P}_{\mathbb{T}}(A):=\overline{P}_2\overline{P}_2^TA\overline{P}_2\overline{P}_2^T \quad {\rm and} \quad {\cal P}_{\mathbb{T}^{\perp}}(A):=\overline{P}_1\overline{P}_1^TA+A\overline{P}_1\overline{P}_1^T-\overline{P}_1\overline{P}_1^TA\overline{P}_1\overline{P}_1^T, \quad A\in\mathbb{S}^n.
\end{equation}
Moreover, we have the following orthogonal decomposition
\begin{equation}\label{eq:decomp_T}
	A={\cal P}_{\mathbb{T}}(A)+{\cal P}_{\mathbb{T}^{\perp}}(A) \quad {\rm and} \quad \langle {\cal P}_{\mathbb{T}}(A),{\cal P}_{\mathbb{T}^{\perp}}(B) \rangle =0  \quad \forall\, A,B\in{\mathbb S}^n.
\end{equation}
It then follows from the definition of ${\cal P}_{\mathbb{T}}$ that for any $A\in{\mathbb S}^n$, ${\cal P}_{\mathbb{T}^{\perp}}(A)=\overline{P}_1\overline{P}_1^TA+\overline{P}_2\overline{P}_2^TA\overline{P}_1\overline{P}_1^T$,
which implies that ${\rm rank}({\cal P}_{\mathbb{T}^{\perp}}(A))\leq 2r$. This yields that for any $A\in{\mathbb S}^n$,
\begin{equation}\label{eq:PT_bound}
	\|{\cal P}_{\mathbb{T}^{\perp}}(A)\|_*\le \sqrt{2r}\|A\|.
\end{equation}

For any given $S\in {\mathbb S}^n$, we use ${\rm supp}_S$ to denote the index set of the support of $S$, i.e., 
$${\rm supp}(S):=\left\{(i,j)\mid S_{ij} \neq 0, i,j\in\{1,\ldots,n\}\right\}.$$
Suppose that the unknown matrix $\overline{S}$ has $k$ nonzero entries, i.e., $\|\overline{S}\|_0=|{\rm supp}(\overline{S})|=k$. The tangent subspace $\mathbb{L}$ with respect to the $k$-sparse subset $\{S\in {\mathbb S}^n\ |\ \|S\|_0\leq k\}$ at $\overline{S}$ then takes the form 
$$
\mathbb{L}:=\{S\in {\mathbb S}^n\ |\ {\rm supp}(S) \subseteq {\rm supp}(\overline{S})\}.
$$ 
Denote the orthogonal complement of $\mathbb{L}$ by $\mathbb{L}^{\perp}$. Let ${\cal P}_{\mathbb{L}}$ and ${\cal P}_{\mathbb{L}^{\perp}}$ be the corresponding orthogonal projections  onto $\mathbb{L}$ and $\mathbb{L}^{\perp}$. Then, we have the following decomposition 
\begin{equation}\label{eq:decomp_L}
	B={\cal P}_{\mathbb{L}}(B)+{\cal P}_{\mathbb{L}^{\perp}}(B) \quad {\rm and} \quad \langle {\cal P}_{\mathbb{L}}(A),{\cal P}_{\mathbb{L}^{\perp}}(B) \rangle =0  \quad \forall\, A,B\in{\mathbb S}^n.
\end{equation} 
Moreover, for any $B\in\mathbb{S}^n$, since $\|{\cal P}_{\mathbb{L}}(B)\|_0 \leq k$, we have that
\begin{eqnarray}\label{eq:PGamma_bound} 
	\|{\cal P}_{\mathbb{L}}(B)\|_1 \leq \sqrt{k}\|B\|.
\end{eqnarray}
Define 
\begin{eqnarray}\label{def:apprx_distance}
	a_D:=\frac{1}{\sqrt {r}}\|\overline{P}_1\overline{P}_1^T-\widetilde{F}\|\quad {\rm and}\quad a_S:=\frac{1}{\sqrt k}\|{\rm sgn}(\overline{S})-\widetilde{G}\|.
\end{eqnarray}

It is also easy to verify the following result.
\begin{lemma}\label{lm:rank}
	For any $D\in\mathbb{S}^n_h$, we have $D-JDJ=\frac{1}{2}({\rm diag} (-JDJ){\bf 1}^T+{\bf 1}{\rm diag}(-JDJ)^T)$.
\end{lemma}

The following result represents the first important step to derive our error bounds of the convex model \eqref{pr:cvxopt}. 

\begin{proposition}\label{prop:sampling_bound}
	Let $(\widehat{D}, \widehat{S})$ and $(\overline{D}, \overline{S})$ be an optimal solution of \eqref{pr:cvxopt} and the underground true EDM and outlier matrices, respectively. Let $\kappa_D > 1$ and $\kappa_S>1$ be  given. Suppose that the  parameters $\rho_D$ and $\rho_S$ satisfy
	\begin{equation}\label{eq:PenalizedPara}
		\rho_D\geq \frac{\kappa_D}{m}\big\| \mathcal{O}^*_{\Omega}(\zeta)\big\|_2  \quad {\rm and} \quad 
		\rho_S\geq \frac{\kappa_S}{m}\big\| \mathcal{O}^*_{\Omega}(\zeta)\big\|_\infty,
	\end{equation}
	where $\zeta:={\cal O}_{\Omega}(\daleth)\in\mathbb{R}^m$ and $\daleth\in\mathbb{S}^n$ is given by \eqref{eq:def-daleth}, then we have
	\begin{equation}\label{eq:sampling_bound-1}
		\frac{1}{2m}\|{\mathcal O}_{\Omega} (\widehat{D}-\overline{D})+{\mathcal O}_{\Omega} (\widehat{S}-\overline{S})\|^2\leq\rho_D\sqrt{r}(a_D+\frac{2\sqrt{2}}{\kappa_D})\|\widehat{D}-\overline{D}\|+\rho_S\sqrt{k}(a_S+\frac{1}{\kappa_S}) \|\widehat{S}-\overline{S}\| 
	\end{equation}
	and
	\begin{equation}\label{eq:sampling_bound-2}
		\left\{
		\begin{array}{l}
			\|\widehat{D}-\overline{D}\|_* \leq \displaystyle \frac{\kappa_D}{\kappa_D-1} \left((a_D +2\sqrt{2})\sqrt{r} \|\widehat{D}-\overline{D}\|+\frac{\rho_S}{\rho_D} \big( a_S+\frac{1}{\kappa_S}\big)\sqrt{k}\|\widehat{S}-\overline{S}\| \right),  \\[3pt]
			\|\widehat{S}-\overline{S}\|_1 \leq \displaystyle \frac{\kappa_S}{\kappa_S-1}\left(\frac{\rho_D}{\rho_S} \big(a_D +\frac{2\sqrt{2}}{\kappa_D}\big)\sqrt{r}\|\widehat{D}-\overline{D}\|+\left( a_S+1\right)\sqrt{k}\|\widehat{S}-\overline{S}\| \right),
		\end{array}\right. 
	\end{equation}
	where $a_D$ and $a_S$ are given by \eqref{def:apprx_distance}.
\end{proposition}
\begin{proof}
	By \eqref{eq:def-ob-y} and \eqref{eq:def-D-tilde}, we know that for any $D$ and $S\in\mathbb{S}^n$, 
	\begin{eqnarray}
		\frac{1}{2m}\|{\bf y} -{\cal O}_{\Omega}(D+S)\|^2&=&\frac{1}{2m}\big\| {\cal O}_{\Omega}(\overline{D}+\overline{S})+\zeta-{\cal O}_{\Omega}(D+S) \big\|^{2} \nonumber \\ 
		&=& \frac{1}{2m}\|{\cal O}_{\Omega}(D-\overline{D})+{\cal O}_{\Omega}(S-\overline{S})\|^2-\frac{1}{m}\langle D-\overline{D}+S-\overline{S},{\cal O}_{\Omega}^*(\zeta)\rangle + \frac{1}{2m}\|\zeta\|^2. \label{eq:diff-eq}
	\end{eqnarray}
	
	Since $(\overline{D},\overline{S})$ is a feasible solution of \eqref{pr:cvxopt}, we know from the optimality of $(\widehat{D},\widehat{S})$ that 
	\begin{eqnarray}
		\frac{1}{2m}\|{\cal O}_{\Omega}(\widehat{D}-\overline{D})+{\cal O}_{\Omega}(\widehat{S}-\overline{S})\|^2 &\leq & \frac{1}{m} \left\langle {\cal O}_{\Omega}^*(\zeta), \widehat{D}-\overline{D}+\widehat{S}-\overline{S}\right\rangle-\rho_D \left(\langle I, -J (\widehat{D}-\overline{D}) J\rangle  - \langle \widetilde{F}, -J(\widehat{D}-\overline{D})J\rangle \right) \nonumber \\ [3pt]
		&& - \rho_S \left(\|\widehat{S}\|_1-\|\overline{S}\|_1- \langle \widetilde{G}, \widehat{S}-\overline{S} \rangle \right). \label{eq:ineq-1}
	\end{eqnarray}
	By the H$\ddot{\rm o}$lder inequality, we know that the first term of the right hand side of \eqref{eq:ineq-1} satisfies
	\begin{equation*}
		\frac{1}{m} \left\langle {\cal O}_{\Omega}^*(\zeta), \widehat{D}-\overline{D}+\widehat{S}-\overline{S}\right\rangle 
		\leq \frac{1}{m}\left\|\mathcal{O}_{\Omega}^\ast(\zeta)\right\|_2 \|\widehat{D}-\overline{D}\|_\ast+\frac{1}{m}\left\|\mathcal{O}_{\Omega}^\ast(\zeta)\right\|_\infty \|\widehat{S}-\overline{S}\|_1.
	\end{equation*}
	Since $\|\widehat{D}-\overline{D}\|_*=\|\widehat{D}-\overline{D}-J(\widehat{D}-\overline{D})J+J(\widehat{D}-\overline{D})J\|_*$, we know from Lemma \ref{lm:rank} that
	\begin{equation*}
		\|\widehat{D}-\overline{D}\|_*\le  \|\widehat{D}-\overline{D}-J(\widehat{D}-\overline{D})J\|_*+\|J(\widehat{D}-\overline{D})J\|_*\le \sqrt{2}\|\widehat{D}-\overline{D}-J(\widehat{D}-\overline{D})J\|+\|J(\widehat{D}-\overline{D})J\|_*.
	\end{equation*}
	Moreover, since $\left\langle J (\widehat{D}-\overline{D}) J,\, \widehat{D}-\overline{D} - J (\widehat{D}-\overline{D}) J \right\rangle=0$, we have $
	\|\widehat{D}-\overline{D}\|^2=\|\widehat{D}-\overline{D} - J (\widehat{D}-\overline{D}) J \|^2+\|J (\widehat{D}-\overline{D}) J\|^2$, which implies that $\|\widehat{D}-\overline{D}\|_*\le \sqrt{2}\|\widehat{D}-\overline{D}\|+\|-J(\widehat{D}-\overline{D})J\|_*$. Thus, since $\|J (\widehat{D}-\overline{D}) J\|\le \|\widehat{D}-\overline{D}\|$, by \eqref{eq:PenalizedPara} and \eqref{eq:PT_bound},  we know that 
	\begin{eqnarray*}
		\frac{1}{m}\left\|\mathcal{O}_{\Omega}^\ast(\zeta)\right\|_2 \|\widehat{D}-\overline{D}\|_\ast &\le & \frac{1}{m}\left\|\mathcal{O}_{\Omega}^\ast(\zeta)\right\|_2\left( \sqrt{2}\|\widehat{D}-\overline{D}\|+\|-J(\widehat{D}-\overline{D})J\|_*\right) \\ [3pt]
		&\le & \frac{1}{m}\left\|\mathcal{O}_{\Omega}^\ast(\zeta)\right\|_2\left( \sqrt{2}\|\widehat{D}-\overline{D}\|+\|{\cal P}_{\mathbb{T}}(-J(\widehat{D}-\overline{D})J)\|_*+\|{\cal P}_{\mathbb{T}^{\perp}}(-J(\widehat{D}-\overline{D})J)\|_*\right) \\ [3pt]
		&\le & \frac{\rho_D}{\kappa_D}\left( \sqrt{2}\|\widehat{D}-\overline{D}\|+\|{\cal P}_{\mathbb{T}}(-J(\widehat{D}-\overline{D})J)\|_*+\|{\cal P}_{\mathbb{T}^{\perp}}(-J(\widehat{D}-\overline{D})J)\|_*\right) \\ [3pt]
		&\le & \frac{\rho_D}{\kappa_D}\left( \sqrt{2}\|\widehat{D}-\overline{D}\|+\|{\cal P}_{\mathbb{T}}(-J(\widehat{D}-\overline{D})J)\|_*+\sqrt{2r}\|J(\widehat{D}-\overline{D})J\|\right) \\ [3pt]
		&\le & \frac{\rho_D}{\kappa_D}\left( (\sqrt{2}+\sqrt{2r})\|\widehat{D}-\overline{D}\|+\|{\cal P}_{\mathbb{T}}(-J(\widehat{D}-\overline{D})J)\|_*\right),
	\end{eqnarray*}
	Similarly, we know from \eqref{eq:PGamma_bound} and \eqref{eq:PenalizedPara} that
	\[
	\frac{1}{m}\left\|\mathcal{O}_{\Omega}^\ast(\zeta)\right\|_\infty \|\widehat{S}-\overline{S}\|_1\le \frac{\rho_S}{\kappa_S} \left(\|{\cal P}_{\mathbb{L}}(\widehat{S}-\overline{S})\|_1+\|{\cal P}_{\mathbb{L}^{\perp}}(\widehat{S}-\overline{S})\|_1 \right)\le \frac{\rho_S}{\kappa_S} \left(\sqrt{k}\|\widehat{S}-\overline{S}\|+\|{\cal P}_{\mathbb{L}^{\perp}}(\widehat{S}-\overline{S})\|_1 \right).
	\]
	Therefore, we obtain that the first term of the right hand side of \eqref{eq:ineq-1} satisfies 
	\begin{eqnarray}
		\frac{1}{m} \left\langle {\cal O}_{\Omega}^*(\zeta), \widehat{D}-\overline{D}+\widehat{S}-\overline{S}\right\rangle &\leq &\frac{\rho_D}{\kappa_D}\left( \sqrt{2}\|\widehat{D}-\overline{D}\|+\|{\cal P}_{\mathbb{T}}(-J(\widehat{D}-\overline{D})J)\|_*+\|{\cal P}_{\mathbb{T}^{\perp}}(-J(\widehat{D}-\overline{D})J)\|_*\right) \nonumber \\ [3pt]
		&&+\frac{\rho_S}{\kappa_S} \left(\sqrt{k}\|\widehat{S}-\overline{S}\|+\|{\cal P}_{\mathbb{L}^{\perp}}(\widehat{S}-\overline{S})\|_1 \right). \label{eq:ineq-1-term1}
	\end{eqnarray}
	
	Meanwhile, since for any $A\in{\mathbb S}^n$, $\|{\cal P}_{\mathbb T}(A)\|_*=\|\overline{P}_2^TA\overline{P}_2\|_*$ and both $-J\widetilde{D}J$ and $-J\overline{D}J$ are positively semidefinite, we know from the directional derivative formula of the nuclear norm \cite[Theorem 1]{Watson92} 
	that
	\begin{eqnarray*}
		\langle I, -J (\widehat{D}-\overline{D}) J\rangle=\|-J\widehat{D}J\|_*-\|-J\overline{D}J\|_*&\ge & \langle \overline{P}_1\overline{P}_1^T,-J(\widetilde{D}-\overline{D})J\rangle+\|\overline{P}_2^T(-J(\widetilde{D}-\overline{D})J)\overline{P}_2\|_*\\ 
		&=&\langle \overline{P}_1\overline{P}_1^T,-J(\widetilde{D}-\overline{D})J\rangle+\|{\cal P}_{\mathbb T}(-J(\widetilde{D}-\overline{D})J)\|_*, 
	\end{eqnarray*}
	which implies that the second term of the right hand side of \eqref{eq:ineq-1} satisfies 
	\begin{eqnarray*}
		-\rho_D \left(\langle I, -J (\widehat{D}-\overline{D}) J\rangle  - \langle \widetilde{F}, -J(\widehat{D}-\overline{D})J\rangle \right)&\le & -\rho_D \left( \langle \overline{P}_1\overline{P}_1^T-\widetilde{F},-J(\widetilde{D}-\overline{D})J\rangle+\|{\cal P}_{\mathbb T}(-J(\widetilde{D}-\overline{D})J)\|_* \right). 
	\end{eqnarray*}
	By using the decomposition \eqref{eq:decomp_T} and the notations defined in \eqref{def:apprx_distance}, we obtain that
	\begin{eqnarray}
		-\rho_D \left(\langle I, -J (\widehat{D}-\overline{D}) J\rangle  - \langle \widetilde{F}, -J(\widehat{D}-\overline{D})J\rangle \right) 
		&\leq& \rho_D\left( \|\overline{P}_1\overline{P}_1^T-\widetilde{F}\| \|J(\widehat{D}-\overline{D})J\| -\|{\cal P}_{\mathbb T}(-J(\widehat{D}-\overline{D})J)\|_* \right) \nonumber \\ [3pt]
		&\leq& \rho_D\left( a_D \sqrt{r}\|\widehat{D}-\overline{D}\|-\|{\cal P}_{\mathbb T}(-J(\widehat{D}-\overline{D})J)\|_*\right). \label{eq:ineq-1-term2}
	\end{eqnarray}
	Similarly, we know from the directional derivative of the $l_1$-norm at $\overline{S}$ that 
	\begin{eqnarray*}
		\|\widehat{S}\|_1-\|\overline{S}\|_1\geq \langle \textrm{sign}(\overline{S}), \widehat{S}-\overline{S}\rangle +\|\mathcal{P}_{\mathbb{L}^\perp}(\widehat{S}-\overline{S})\|_1.
	\end{eqnarray*}
	Therefore, by the decomposition \eqref{eq:decomp_L} and the notations defined in \eqref{def:apprx_distance}, we know that the third term of the right hand side of \eqref{eq:ineq-1} satisfies 
	\begin{eqnarray}
		-\rho_S\left(\|\widehat{S}\|_1-\|\overline{S}\|_1- \langle \widetilde{G}, \widehat{S}-\overline{S} \rangle\right) &\leq& -\rho_S\left(\langle \textrm{sign}(\overline{S}), \widehat{S}-\overline{S}\rangle +\|\mathcal{P}_{\mathbb{L}^\perp}(\widehat{S}-\overline{S})\|_1 -\langle \widetilde{G}, \widehat{S}-\overline{S}\rangle \right) \nonumber \\ [3pt]
		&\leq&\rho_S\left(\|\textrm{sign}(\overline{S})-\widetilde{G}\|\|\widehat{S}-\overline{S}\|-\|\mathcal{P}_{\mathbb{L}^\perp}(\widehat{S}-\overline{S})\|_1 \right) \nonumber \\ [3pt]
		&\leq & \rho_S\left( a_S\sqrt{k} \|\widehat{S}-\overline{S}\| - \|\mathcal{P}_{\mathbb{L}^\perp}(\widehat{S}-\overline{S})\|_1\right). \label{eq:ineq-1-term3}
	\end{eqnarray}
	Finally, by substituting \eqref{eq:ineq-1-term1}, \eqref{eq:ineq-1-term2} and \eqref{eq:ineq-1-term3} into \eqref{eq:ineq-1}, we obtain that
	\begin{eqnarray}
		&&\frac{1}{2m}\|{\cal O}_{\Omega}(\widehat{D}-\overline{D})+{\cal O}_{\Omega}(\widehat{S}-\overline{S})\|^2 \nonumber \\ [3pt]
		&\leq& \frac{\rho_D}{\kappa_D}\left( (\sqrt{2}+\sqrt{2r})\|\widehat{D}-\overline{D}\|+\|{\cal P}_{\mathbb{T}}(-J(\widehat{D}-\overline{D})J)\|_*\right) +\frac{\rho_S}{\kappa_S} \left(\sqrt{k}\|\widehat{S}-\overline{S}\|+\|{\cal P}_{\mathbb{L}^{\perp}}(\widehat{S}-\overline{S})\|_1 \right) \nonumber \\ [3pt]
		&&+ \rho_D\left( a_D \sqrt{r}\|\widehat{D}-\overline{D}\|-\|{\cal P}_{\mathbb T}(-J(\widehat{D}-\overline{D})J)\|_*\right)+\rho_S\left( a_S\sqrt{k} \|\widehat{S}-\overline{S}\| - \|\mathcal{P}_{\mathbb{L}^\perp}(\widehat{S}-\overline{S})\|_1\right) \nonumber \\ [3pt]
		&=& \rho_D \left(\frac{1}{\kappa_D}(\sqrt{2}+\sqrt{2r})+a_D\sqrt{r}\right)\|\widehat{D}-\overline{D}\| + \rho_S\left(\frac{1}{\kappa_S}+a_S\right)\sqrt{k}\|\widehat{S}-\overline{S}\| \nonumber \\ [3pt]
		&& -\rho_D\frac{\kappa_D-1}{\kappa_D}\|{\cal P}_{\mathbb{T}}(-J(\widehat{D}-\overline{D})J)\|_* -\rho_S\frac{\kappa_S-1}{\kappa_S}\|\mathcal{P}_{\mathbb{L}^\perp}(\widehat{S}-\overline{S})\|_1. \label{eq:bounds-part1}
	\end{eqnarray}
	Since $r\ge1$, together with the assumptions that $\kappa_D>1$ and $\kappa_S>1$, we know that the inequality \eqref{eq:sampling_bound-1} holds.
	
	Next, we shall show that the inequalities \eqref{eq:sampling_bound-2} also hold. By \eqref{eq:bounds-part1}, we have
	\[
	\left\{\begin{array}{l}
		\|{\cal P}_{\mathbb{T}}(-J(\widehat{D}-\overline{D})J)\|_*\le \displaystyle \frac{\kappa_D}{\kappa_D-1}\left(\sqrt{r} \left(a_D +\frac{2\sqrt{2}}{\kappa_D}\right)\|\widehat{D}-\overline{D}\|+\sqrt{k}\frac{\rho_S}{\rho_D} \left( a_S+\frac{1}{\kappa_S}\right)\|\widehat{S}-\overline{S}\| \right), \\ [3pt]
		\|\mathcal{P}_{\mathbb{L}^\perp}(\widehat{S}-\overline{S})\|_1\le \displaystyle \frac{\kappa_S}{\kappa_S-1}\left(\sqrt{r}\frac{\rho_D}{\rho_S} \left(a_D +\frac{2\sqrt{2}}{\kappa_D}\right)\|\widehat{D}-\overline{D}\|+\sqrt{k} \left( a_S+\frac{1}{\kappa_S}\right)\|\widehat{S}-\overline{S}\| \right).
	\end{array}\right.
	\]
	Therefore, we know from \eqref{eq:decomp_T}, \eqref{eq:decomp_L}, \eqref{eq:PT_bound} and \eqref{eq:PGamma_bound} that
	\begin{eqnarray*}
		\|\widehat{D}-\overline{D}\|_*&\le& \|\widehat{D}-\overline{D}-J(\widehat{D}-\overline{D})J\|_*+\|{\cal P}_{\mathbb{T}^\perp}(-J(\widehat{D}-\overline{D})J)\|_*+\|{\cal P}_{\mathbb{T}}(-J(\widehat{D}-\overline{D})J)\|_* \\ [3pt]
		&\le& (\sqrt{2}+\sqrt{2r})\|\widehat{D}-\overline{D}\| + \frac{\kappa_D}{\kappa_D-1}\left(\sqrt{r} \big(a_D +\frac{2\sqrt{2}}{\kappa_D}\big)\|\widehat{D}-\overline{D}\|+\sqrt{k}\frac{\rho_S}{\rho_D} \left( a_S+\frac{1}{\kappa_S}\right)\|\widehat{S}-\overline{S}\| \right) \\ [3pt]
		&\le& \frac{\kappa_D}{\kappa_D-1} \left((a_D +2\sqrt{2})\sqrt{r} \|\widehat{D}-\overline{D}\|+\frac{\rho_S}{\rho_D} \big( a_S+\frac{1}{\kappa_S}\big)\sqrt{k}\|\widehat{S}-\overline{S}\| \right)
	\end{eqnarray*}
	and
	\begin{eqnarray*}
		\|\widehat{S}-\overline{S}\|_1&\le& \|{\cal P}_{\mathbb{L}}(\widehat{S}-\overline{S})\|_1+\|{\cal P}_{\mathbb{L}^\perp}(\widehat{S}-\overline{S})\|_1 \\ [3pt]
		&\le& \sqrt{k}\|\widehat{S}-\overline{S}\| + \frac{\kappa_S}{\kappa_S-1}\left(\sqrt{r}\frac{\rho_D}{\rho_S} \left(a_D +\frac{2\sqrt{2}}{\kappa_D}\right)\|\widehat{D}-\overline{D}\|+\sqrt{k} \left( a_S+\frac{1}{\kappa_S}\right)\|\widehat{S}-\overline{S}\| \right) \\ [3pt]
		&\le& \frac{\kappa_S}{\kappa_S-1}\left(\frac{\rho_D}{\rho_S} \big(a_D +\frac{2\sqrt{2}}{\kappa_D}\big)\sqrt{r}\|\widehat{D}-\overline{D}\|+\left( a_S+1\right)\sqrt{k}\|\widehat{S}-\overline{S}\| \right).
	\end{eqnarray*}
	This completes the proof.  \hfill $\Box$ 
\end{proof}

Since $X_1,\ldots,X_m$ are the i.i.d. random observations, i.e., for any $1\le i<j \le n$,
\[
\mathbb{P}\left(X_{l}=\frac{1}{2}({\bf e}_i {\bf e}_j^T + {\bf e}_j {\bf e}_i^T)\right)=\pi_{ij},\quad l=1,\ldots,m,
\]
where  $0\le \pi_{ij}\le 1$ is the probability that the $(i,j)$ and $(j,i)$-th element be sampled in the observation model. We propose the following assumption to control the sampling probability.  
\begin{assumption}\label{as:prob_lower_bound}
	There exist two absolution constants $\mu_1, \mu_2\ge 1$ such that 
	\begin{equation*}
		\frac{1}{\mu_1d_{\mathbb{S}^n_h}}\le \pi_{ij}\le \frac{\mu_2}{d_{\mathbb{S}^n_h}} \quad  \forall\ 1\le i<j\le n,
	\end{equation*}
	where $d_{\mathbb{S}^n_h}=n(n-1)/2$.
\end{assumption}

It is easy to see from Assumption \ref{as:prob_lower_bound} that for any $A\in{\mathbb S}^n_h$, we have
\begin{equation}\label{eq:E2-bounded}
	{\mathbb E}\left(\langle A,X \rangle^2 \right)\ge \frac{1}{2\mu_1 d_{\mathbb{S}^n_h}}\|A\|^2.
\end{equation} 
Furthermore, let $m_{\max}$ be the maximum number of repetitions of any $(i,j)$ index in $\Omega$. By noting the sample size $m$ is assumed much smaller than $d_{\mathbb{S}^n_h}$, we obtain from \cite[Lemma 5.5]{Wu2014}  the following result on the upper bound of $m_{\max}$. For simplicity, we omit the detailed proof here.

\begin{lemma}\label{lem:max_num_repetition}
	Let the observation index set $\Omega$ be generated by the uniform sampling with replacement. Then, there exists a constant $C>0$ such that
	\[
	m_{\max}\le \|\mathcal{O}_{\Omega}^*\mathcal{O}_{\Omega}\|_2\le C\log(2n^2)
	\] 
	with probability at least $1-1/(2n^2)$.
	
\end{lemma}

We further introduce the following two useful notations:
\begin{eqnarray}\label{def:exp_inv_operator}
	\vartheta_D :=\mathbb{E}\left\|\frac{1}{m}\mathcal{O}_\Omega^*(\epsilon)\right\|_2 \quad  {\rm and} 
	\quad \vartheta_S:=\mathbb{E}\left\|\frac{1}{m}\mathcal{O}_\Omega^*(\epsilon)\right\|_\infty,
\end{eqnarray}
where $\{\epsilon_1, \ldots, \epsilon_m\}$ is a Rademacher sequence, i.e., an i.i.d. copy  of Bernoulli random variable  taking the values $1$ and $-1$ with probability $1/2$.

For the given positive numbers $p_1$, $p_2$, $q_1$, $q_2$ and $t$, define the following subset $K(p_1,p_2,q_1,q_2,t)\subseteq \mathbb{S}^n$ by
\begin{equation}\label{eq:def-K-set}
	K({\bf p},{\bf q},t):=
	\left \{ A=A_D + A_S \mid
	\begin{array}{l}
		\|A_D\|_*\leq p_1\|A_D\| + p_2\|A_S\|,\ A_D\in \mathbb{S}^n_h, \\
		\|A_S\|_1\leq q_1\|A_D\| + q_2\|A_S\|,\ A_S\in \mathbb{S}^n_h, \\
		\|A_D + A_S\|_\infty=1,\ \|A_D\|^2+\|A_S\|^2\geq t\mu_1 d_{\mathbb{S}^n_h}
	\end{array}
	\right \},
\end{equation}
where ${\bf p}:=(p_1,p_2)$ and ${\bf q}:=(q_1,q_2)$.
Denote $\vartheta_m:=(\vartheta_D^2 p_1^2+\vartheta_D^2 p_2^2+\vartheta_S^2 p_1^2+\vartheta_S^2 p_2^2)$.

\begin{proposition}\label{prop:error-1}
	Suppose that Assumption \ref{as:prob_lower_bound} holds. Let $p_1$, $p_2$, $q_1$, $q_2$ and $t$ be any given positive numbers. For any $\tau_1$ and $\tau_2$ satisfying
	\[ 0 <\tau_1< 1 \quad {\rm and} \quad 0<\tau_2< \frac{\tau_1}{2}, \]
	it holds that for any $A\in K({\bf p},{\bf q},t)$,
	\begin{equation}\label{eq:error-1}
		\frac{1}{m}\|{\cal O}_\Omega(A)\|^2\geq {\mathbb E}\left(\langle A,X \rangle^2 \right) - \frac{\tau_1}{\mu_1 d_{\mathbb{S}^n_h}}\left(\|A_D\|^2+\|A_S\|^2\right) - \frac{32}{\tau_2}\mu_1d_{\mathbb{S}^n_h} \vartheta_m^2
	\end{equation}
	with probability at least 
	$$1-\frac{\exp(-(\tau_1-2\tau_2)^2 mt^2/8)}{1-\exp(-3(\tau_1-2\tau_2)^2 mt^2/8)}.$$
\end{proposition}
\begin{proof}
	We will show that the event 
	\begin{equation*}
		E:=\left\{  \exists\, A\in K({\bf p},{\bf q},t)\ \textrm{s.t.}\ 
		\left|\frac{1}{m}\|\mathcal{O}_\Omega(A)\|^2-{\mathbb E}\left(\langle A,X \rangle^2 \right)\right| \geq \frac{\tau_1}{\mu_1 d_{\mathbb{S}^n_h}}\left(\|A_D\|^2+\|A_S\|^2\right)+ \frac{32}{\tau_2}\mu_1d_{\mathbb{S}^n_h} \vartheta_m^2
		\right\}
	\end{equation*}
	happens with probability less than $\displaystyle\frac{\exp[-(\tau_1-2\tau_2)^2 mt^2/32]}{1-\exp[-3(\tau_1-2\tau_2)^2 mt^2/32]}$. First, we decompose  $K({\bf p},{\bf q},t)$ by 
	\begin{equation*}
		K({\bf p},{\bf q},t)=\bigcup_{j=1}^{\infty}\left\{A\in K({\bf p},{\bf q},t) \mid 2^{j-1}t\leq \frac{1}{\mu_1 d_{\mathbb{S}^n_h}}\left(\|A_D\|^2+\|A_S\|^2\right)\leq 2^j t \right\}.
	\end{equation*}
	For any $s\geq t$, define the sunset $\widetilde{K}({\bf p},{\bf q},t,s)\subseteq K({\bf p},{\bf q},t)$ by 
	$$
	\widetilde{K}({\bf p},{\bf q},t,s):=\left\{A\in K({\bf p},{\bf q},t)\mid \frac{1}{\mu_1 d_{\mathbb{S}^n_h}}\left(\|A_D\|^2+\|A_S\|^2\right)\leq s  \right\}.
	$$
	Furthermore, for $j=1,2,\ldots$, let $E_j$ be the set defined by
	\[
	E_j:=\left\{  \exists \, A\in \widetilde{K}({\bf p},{\bf q},t,2^j t)\ \textrm{such that}\ 
	\left|\frac{1}{m}\|\mathcal{O}_\Omega(A)\|^2-{\mathbb E}\left(\langle A,X \rangle^2 \right)\right| \geq \tau_12^{j-1}t + \frac{32}{\tau_2}\mu_1d_{\mathbb{S}^n_h} \vartheta_m^2 \right\}.
	\]
	Then, it is not difficult to see that $E \subseteq \cup_{j=1}^{\infty}E_j$. Thus, it suffices to estimate the probability of each simpler event $E_j$ and then obtain the estimated probability bound of the event $E$. Denote
	\begin{equation*}
		Z_s:=\sup_{A\in \widetilde{K}({\bf p},{\bf q},t,s)} \left|\frac{1}{m}\|\mathcal{O}_\Omega(A)\|^2-{\mathbb E}\left(\langle A,X \rangle^2 \right)\right|.
	\end{equation*}
	For any $A\in \mathbb{S}_h^n$, the strong laws of large numbers yield that
	\begin{equation*}
		\frac{1}{m}\|\mathcal{O}_\Omega(A)\|^2=\frac{1}{m}\sum_{l=1}^{m}\langle X_{l}, A \rangle^2 \xrightarrow{a.s.} \mathbb{E}\left(\langle  A,X \rangle^2\right)  \quad \mbox{as} \quad m\to \infty.
	\end{equation*}
	Since $\|A\|_\infty=1$ for all $A\in K({\bf p},{\bf q},t)$, we know that for any $1\leq l \leq m$ and $A\in K({\bf p},{\bf q},t)$,
	\begin{equation*}
		\left|\langle X_{l}, A \rangle^2-\mathbb{E} \left(\langle X_{l}, A\rangle^2\right)\right| 
		\leq \max \left\{\langle X_{l}, A \rangle^2,\ \mathbb{E} \left(\langle X_{l}, A\rangle^2\right)\right\}
		\leq 1.
	\end{equation*}
	Then, according to Massart's Hoeffding-type concentration inequality \cite[Theorem 14.2]{bhlmann2011statistics} (see also \cite[Theorem 9]{massart2000about}), we know that 
	\begin{equation}\label{eq:MassartHoeffdingIneq}
		\mathbb{P}\left(Z_s\leq \mathbb{E}(Z_s)+\eta\right) \leq \exp\left(-\frac{m\eta^2}{8}\right)  \quad \forall\ \eta \geq 0.
	\end{equation}
	Next, we estimate an upper bound of $\mathbb{E}\left(Z_s\right)$ by using the standard Rademacher symmetrization in the theory of empirical processes. Recall that $\{\epsilon_1, \ldots, \epsilon_m\}$ is a Rademacher sequence. Then, we have
	\begin{eqnarray*}
		\mathbb{E}\left(Z_s\right)&=&\mathbb{E}\left(\sup_{A\in \widetilde{K}({\bf p},{\bf q},t,s)}\left| \frac{1}{m}\sum_{l=1}^m \langle X_{l}, A\rangle^2 - \mathbb{E}\left[\langle X_{l}, A\rangle^2\right] \right| \right)\leq 2\mathbb{E}\left(\sup_{A\in \widetilde{K}({\bf p},{\bf q},t,s)}\left| \frac{1}{m}\sum_{l=1}^m \epsilon_l\langle X_{l}, A\rangle^2 \right|\right)\\
		&\leq& 8\mathbb{E}\left(\sup_{A\in \widetilde{K}({\bf p},{\bf q},t,s)}\left| \frac{1}{m}\sum_{l=1}^m \epsilon_l\langle X_{l}, A\rangle \right|\right)
		= 8\mathbb{E}\left(\sup_{A\in \widetilde{K}({\bf p},{\bf q},t,s)}\left| \left\langle  \frac{1}{m}\mathcal{O}^* (\epsilon), A\right\rangle \right|\right) \\
		&\leq& 8\mathbb{E}\left(\sup_{A\in \widetilde{K}({\bf p},{\bf q},t,s)}\left( \left\| \frac{1}{m}\mathcal{O}^* (\epsilon)\right\|_2 \left\|A_D\right\|_*+\left\| \frac{1}{m}\mathcal{O}^* (\epsilon)\right\|_\infty \left\|A_S\right\|_1 \right)\right) \\
		&\leq& 8\mathbb{E}\left(\sup_{A\in \widetilde{K}({\bf p},{\bf q},t,s)} \left\| \frac{1}{m}\mathcal{O}^* (\epsilon)\right\|_2 \left\|A_D\right\|_*+\sup_{A\in \widetilde{K}({\bf p},{\bf q},t,s)}\left\| \frac{1}{m}\mathcal{O}^* (\epsilon)\right\|_\infty \left\|A_S\right\|_1\right)\\
		&\leq& 8\mathbb{E}\left\| \frac{1}{m}\mathcal{O}^* (\epsilon)\right\|_2\left(\sup_{A\in \widetilde{K}({\bf p},{\bf q},t,s)} \left\|A_D\right\|_*\right)+8\mathbb{E}\left\| \frac{1}{m}\mathcal{O}^* (\epsilon)\right\|_\infty\left(\sup_{A\in \widetilde{K}({\bf p},{\bf q},t,s)}\left\|A_S\right\|_1\right),
	\end{eqnarray*}
	where the first inequality is due to the symmetrization theorem \cite[Lemma 2.3.1]{vaart1996weak} or \cite[Theorem 14.3]{bhlmann2011statistics}, and the second inequality follows from the contraction theorem (e.g., \cite[Theorem 4.12]{ledoux1991statistics} and \cite[Theorem 14.4]{bhlmann2011statistics}, ). Notice that for any $u\ge 0$, $v\ge 0$ and $A\in \widetilde{K}({\bf p},{\bf q},t,s)$,
	\begin{eqnarray*}
		u\|A_D\|+v\|A_S\| &\leq& \frac{16m}{\tau_1}(u^2+v^2)+\frac{\tau_1}{64m}\|A\|^2 \leq\frac{16m}{\tau_1}(u^2+v^2)+\frac{1}{64}\tau_1 s,
	\end{eqnarray*}
	where the first inequality is due to the inequality of arithmetic and geometric means. We derive that
	\begin{eqnarray*}
		\mathbb{E}(Z_s) &\leq& 8\left( \sup_{A\in \widetilde{K}({\bf p},{\bf q},t,s)} \vartheta_D (p_1 \|A_D\| +p_2\|A_S\|)+ \sup_{A\in \widetilde{K}({\bf p},{\bf q},t,s)} \vartheta_S (q_1 \|A_D\| +q_2\|A_S\|) \right)  \\ [3pt]
		&\leq& \frac{16}{\tau_1}m (\vartheta_D^2 p_1^2 +\vartheta_D^2 p_2^2+\vartheta_S^2 p_1^2+\vartheta_S^2 p_2^2)+ \frac{\tau_1}{32} s 
		=\frac{16}{\tau_1}m \vartheta_m^2+\frac{\tau_1}{32} s.
	\end{eqnarray*}
	According to \eqref{eq:MassartHoeffdingIneq}, it follows that
	\begin{equation*}
		\mathbb{P}\left(Z_s\geq \frac{16}{\tau_1}m \vartheta_m^2+\frac{\tau_1}{32} s \right)
		\leq \mathbb{P}\left(Z_s\geq \mathbb{E}(Z_s)+(\frac{\tau_1}{2}-\tau_2)s \right) \leq \exp\left( -\left( \frac{\tau_1}{2}-\tau_2\right)^2\frac{ms^2}{8}\right).
	\end{equation*}
	This, together with the choice of $s=2^j t$, implies that $\mathbb{P}(E_j)\leq\exp \left(-\frac{1}{8}2^{2(j-1)}(\tau_1-2\tau_2)^2mt^2\right)$.
	By using the fact that $2^j\geq1+j(2-1)$ for any $j\geq 1$, we obtain that
	\begin{eqnarray*}
		\mathbb{P}(E)&\leq&\sum_{j=1}^\infty \mathbb{P}(E_j) \leq \sum_{j=1}^\infty\exp\left(-\frac{1}{8}2^{2(j-1)}(\tau_1- 2\tau_2)^2 mt^2 \right) \\
		&\leq& \exp\left(-\frac{1}{8}(\tau_1- 2\tau_2)^2 mt^2 \right) \sum_{j=1}^\infty\exp\left(-\frac{1}{8}(2^{2(j-1)}-1)(\tau_1-2\tau_2)^2 mt^2 \right) \\
		&\leq& \exp\left(-\frac{1}{8}(\tau_1-2\tau_2)^2 mt^2 \right) \sum_{j=1}^\infty\exp\left(-\frac{3}{8}(j-1)(\tau_1-2\tau_2)^2 mt^2 \right) .\\
		&=& \frac{\exp(-(\tau_1-2\tau_2)^2 mt^2/8)}{1-\exp(-3(\tau_1-2\tau_2)^2 mt^2/8)}.
	\end{eqnarray*} 
	The proof is then completed. \hfill $\Box$ 
\end{proof}

\begin{proposition}\label{prop:bound_sum_norm}
	Let $(\widehat{D}, \widehat{S})$ and $(\overline{D}, \overline{S})$ be an optimal solution of \eqref{pr:cvxopt} and the underground true EDM and outlier matrices, respectively. Let $\kappa_D > 1$ and $\kappa_S>1$ be given arbitrarily. Suppose that the  parameters $\rho_D>0$ and $\rho_S>0$ are given by \eqref{eq:PenalizedPara}. Under Assumption \ref{as:prob_lower_bound}, there exist some positive absolute constants $C_0$, $C_1$ and $C_2$ such that either  
	\begin{equation}\label{eq:bound_sum_norm-1}
		\frac{\|\widehat{D}-\overline{D}\|^2+\|\widehat{S}-\overline{S}\|^2}{d_{\mathbb{S}^n_h}}\leq C_0\mu_1(b_D+b_S)^2\sqrt{\frac{\log(2n)}{m}}
	\end{equation}
	or
	\begin{eqnarray}
		\frac{\|\widehat{D}-\overline{D}\|^2+\|\widehat{S}-\overline{S}\|^2}{d_{\mathbb{S}^n_h}}&\leq& C_1\mu_1^2 d_{\mathbb{S}^n_h} \left\{ C_2^2 \left[ \rho_D^2 r\left( a_D+\frac{2\sqrt{2}}{\kappa_D}\right)^2 +\rho_S^2 k\left( a_S+\frac{1}{\kappa_S}\right)^2 \right] \right. \nonumber \\ [3pt]
		& &\left. + \vartheta_D^2(b_D+b_S)^2(\frac{\kappa_D}{\kappa_D-1})^2 \left[r\left(a_D +2\sqrt{2}\right)^2+k\frac{\rho_S^2}{\rho_D^2}\left(a_S+\frac{1}{\kappa_S} \right)^2 \right] \right. \nonumber \\ [3pt]
		& &\left. +\max\left\{ \vartheta_S^2(b_D+b_S)^2, \frac{b_D^2}{\mu_1^2 d_{\mathbb{S}^n_h}^2} \right\} (\frac{\kappa_S}{\kappa_S-1})^2 \left[ r\frac{\rho_D^2}{\rho_S^2} \left( a_D + \frac{2\sqrt{2}}{\kappa_D} \right)^2 +k\left(a_S+1 \right)^2 \right] \right\}, \label{eq:bound_sum_norm-2}
	\end{eqnarray}
	with probability at least $1-(4/7)n^{-1}$, where $a_D$ and $a_S$ are given by \eqref{def:apprx_distance} and $\vartheta_D$, and $\vartheta_S$ are defined by \eqref{def:exp_inv_operator}.
\end{proposition}
\begin{proof} 
	Denote $A:=A_D+A_S$ with $A_D:=\widehat{D}-\overline{D}$ and $A_S:=\widehat{S}-\overline{S}$. Let $b:=\|A\|_\infty$ and $t:=\sqrt{\frac{32\log(2n)}{(\tau_1-2\tau_2)^2m}}$, where $c>0$ and $\tau_1$, $\tau_2$ satisfying $0<\tau_2<\frac{\tau_1}{2} <1/2$ are arbitrarily  fixed constants. Consider the following two cases.
	
	{\bf Case 1.} $\left(\|A_D\|^2+\|A_S\|^2\right)< b^2\mu_1d_{\mathbb{S}^n_h}t$. Since $b=\|A\|_\infty\leq \|A_D\|_\infty+\|A_S\|_\infty \leq 2(b_D+b_S)$, we know that there exists a positive constant $C_0$ such that \eqref{eq:bound_sum_norm-1} holds.
	
	{\bf Case 2.} $\left(\|A_D\|^2+\|A_S\|^2\right)\ge b^2\mu_1d_{\mathbb{S}^n_h}t$. By \eqref{eq:sampling_bound-2} in Proposition \ref{prop:sampling_bound}, we know that $A/b\in K({\bf p},{\bf q},t)$, where $K({\bf p},{\bf q},t)$ is the subset defined by \eqref{eq:def-K-set} with ${\bf p}=(p_1,p_2)$ and ${\bf q}=(q_1,q_2)$ are given by
	\begin{equation}\label{eq:parameter}
		\left\{\begin{array}{ll}
			p_1=\displaystyle\frac{\kappa_D}{\kappa_D-1}\left(a_D+2\sqrt{2}\right)\sqrt{r},\ 
			& p_2=\displaystyle\frac{\kappa_D}{\kappa_D-1}\frac{\rho_S}{\rho_D}\left( a_S+\frac{1}{\kappa_S}\right)\sqrt{k}, \\ [3pt]
			q_1= \displaystyle\frac{\kappa_S}{\kappa_S-1}\frac{\rho_D}{\rho_S} \left(a_D +\frac{2\sqrt{2}}{\kappa_D}\right)\sqrt{r},\ 
			&q_2= \displaystyle\frac{\kappa_S}{\kappa_S-1}\left( a_S+1 \right)\sqrt{k}.
		\end{array}\right.
	\end{equation}
	Therefore, it follows from Proposition \ref{prop:error-1} and \eqref{eq:E2-bounded}  that with probability at least $1-(4/7)n^{-1}$,
	\[
	\frac{1}{d_{\mathbb{S}^n_h}}\|A\|^2\le 2\mu_1{\mathbb E}\left(\langle A,X \rangle^2 \right) \le \frac{2\mu_1}{m}\|{\cal O}_\Omega(A)\|^2  + \frac{2\tau_1}{d_{\mathbb{S}^n_h}}\left(\|A_D\|^2+\|A_S\|^2\right) + \frac{64}{\tau_2}\mu^2_1d_{\mathbb{S}^n_h} \vartheta_m^2b^2.
	\]
	By \eqref{eq:sampling_bound-1} in Proposition \ref{prop:sampling_bound}, we obtain that for any $0<\tau_3<(1-2\tau_1)/2$,
	\begin{eqnarray}
		\frac{1}{d_{\mathbb{S}^n_h}}\|A\|^2&\le& 4\mu_1\rho_D\sqrt{r}(a_D+\frac{2\sqrt{2}}{\kappa_D})\|A_D\|+4\mu_1\rho_S\sqrt{k}(a_S+\frac{1}{\kappa_S}) \|A_S\| \nonumber\\[3pt] 
		&& + \frac{2\tau_1}{d_{\mathbb{S}^n_h}}\left(\|A_D\|^2+\|A_S\|^2\right) + \frac{64}{\tau_2}\mu^2_1d_{\mathbb{S}^n_h} \vartheta_m^2b^2 \nonumber \\ [3pt]
		&\le & \frac{4\mu_1^2\rho_D^2rd_{\mathbb{S}_h^n}}{\tau_3}\Big(a_D+\frac{2\sqrt{2}}{\kappa_D}\Big)^2+\frac{\tau_3}{d_{\mathbb{S}_h^n}}\|A_D\|^2+\frac{4\mu_1^2\rho_S^2kd_{\mathbb{S}_h^n}}{\tau_3}\Big(a_S+\frac{1}{\kappa_S}\Big)^2+\frac{\tau_3}{d_{\mathbb{S}_h^n}}\|A_S\|^2 \nonumber \\[3pt] 
		&& + \frac{2\tau_1}{d_{\mathbb{S}^n_h}}\left(\|A_D\|^2+\|A_S\|^2\right) + \frac{64}{\tau_2}\mu^2_1d_{\mathbb{S}^n_h} \vartheta_m^2b^2 \nonumber \\ [3pt]
		&=& \frac{4\mu_1^2\rho_D^2rd_{\mathbb{S}_h^n}}{\tau_3}\Big(a_D+\frac{2\sqrt{2}}{\kappa_D}\Big)^2+\frac{4\mu_1^2\rho_S^2kd_{\mathbb{S}_h^n}}{\tau_3}\Big(a_S+\frac{1}{\kappa_S}\Big)^2 + \frac{2\tau_1+\tau_3}{d_{\mathbb{S}^n_h}}\left(\|A_D\|^2+\|A_S\|^2\right) \nonumber \\ [3pt]
		&& + \frac{64}{\tau_2}\mu^2_1d_{\mathbb{S}^n_h} \vartheta_m^2b^2. \label{eq:error-tp-2}
	\end{eqnarray}
	In addition, since $\|A_D\|\leq 2b_D$, we then derive from that
	\begin{eqnarray*}
		\|A\|^2 &\geq& \|A_D\|^2 + \|A_S\|^2 - 2\|A_D\|_\infty \|A_S\|_1 \geq \|A_D\|^2 + \|A_S\|^2 - 4b_D(q_1\|A_D\|+q_2\|A_S\|)  \\ [3pt]
		&\geq& \|A_D\|^2 + \|A_S\|^2 - \frac{4}{\tau_3}b_D^2(q_1^2+q_2^2)-\tau_3(\|A_D\|^2+\|A_S\|^2).
	\end{eqnarray*}
	This, together with \eqref{eq:error-tp-2}, yields that
	\begin{eqnarray*}
		\frac{1-\tau_3}{d_{\mathbb{S}^n_h}}\left(\|A_D\|^2+\|A_S\|^2\right) &\leq& \frac{1}{d_{\mathbb{S}^n_h}}\|A\|^2 +\frac{4}{d_{\mathbb{S}^n_h}\tau_3}b_D^2(q_1^2+q_2^2) \\ [3pt]
		&\leq&\frac{4\mu_1^2\rho_D^2rd_{\mathbb{S}_h^n}}{\tau_3}\Big(a_D+\frac{2\sqrt{2}}{\kappa_D}\Big)^2+\frac{4\mu_1^2\rho_S^2kd_{\mathbb{S}_h^n}}{\tau_3}\Big(a_S+\frac{1}{\kappa_S}\Big)^2 + \frac{2\tau_1+\tau_3}{d_{\mathbb{S}^n_h}}\left(\|A_D\|^2+\|A_S\|^2\right) \\ [3pt]
		&& +\frac{64}{\tau_2}\mu^2_1d_{\mathbb{S}^n_h} \vartheta_m^2b^2+\frac{4}{d_{\mathbb{S}^n_h}\tau_3}b_D^2(q_1^2+q_2^2)  \\ [3pt]
		&=&\frac{4\mu_1^2d_{\mathbb{S}_h^n}}{\tau_3}\left( \rho_D^2r \Big(a_D+\frac{2\sqrt{2}}{\kappa_D}\Big)^2 +\rho_S^2 k\Big( a_S+\frac{1}{\kappa_S}\Big)^2 \right) + \frac{2\tau_1+\tau_3}{d_{\mathbb{S}^n_h}}\left(\|A_D\|^2+\|A_S\|^2\right) \\ [3pt]
		&&+\frac{64}{\tau_2}\mu^2_1d_{\mathbb{S}^n_h} \vartheta_m^2b^2+\frac{4}{d_{\mathbb{S}^n_h}\tau_3}b_D^2(q_1^2+q_2^2).
	\end{eqnarray*}
	Since $1-2(\tau_1+\tau_3)>0$, we have
	\begin{eqnarray*}
		\frac{\|A_D\|^2+\|A_S\|^2}{d_{\mathbb{S}^n_h}} & \leq & \frac{4\mu_1^2d_{\mathbb{S}_h^n}}{1-2(\tau_1+\tau_3)} \left( \frac{1}{\tau_3}\big( \rho_D^2r (a_D+\frac{2\sqrt{2}}{\kappa_D})^2 +\rho_S^2 k( a_S+\frac{1}{\kappa_S})^2 \big)+\frac{16}{\tau_2} \vartheta_m^2 b^2 \right. \\
		&&  \left. +\frac{1}{d_{\mathbb{S}_h^n}^2 \mu_1^2 \tau_3}b_D^2(q_1^2+q_2^2) \right) . 
	\end{eqnarray*}
	Recall that $\vartheta_m^2=\vartheta_D^2p_1^2+\vartheta_D^2p_2^2+\vartheta_S^2q_1^2+\vartheta_S^2q_2^2$. By plugging this together with (\ref{eq:parameter}) into the above inequality and choosing $\tau_1$, $\tau_2$ and $\tau_3$ to be  constants, we complete the proof. \hfill $\Box$ 
\end{proof}


%

In order to obtain the explicit formulas  of the penalized parameters $\rho_D$ and $\rho_S$ based on \eqref{eq:PenalizedPara}, we shall derive the  probabilistic upper bounds on the terms $\frac{1}{m} \|{\cal O}_\Omega^*(\zeta)\|_2$ and $\frac{1}{m} \|{\cal O}_\Omega^*(\zeta)\|_\infty$. To this end, similar with \cite{DQi2017}, from now on, we always assume that the i.i.d. random noises $\xi_l$, $l=1,\ldots,m$ in the sampling model \eqref{eq:distance_estimation} satisfy the following sub-Gaussian tail condition.

\begin{assumption}\label{ass:sub-Gaussian}
	There exist positive constants $K_1$ and $K_2$ such that for all $t>0$,
	\[
	{\mathbb P}\left( |\xi_l|\geq t\right)\leq K_1{\rm exp}\left( - t^2/K_2\right).
	\]
\end{assumption}

The following proposition on the upper bounds on the terms $\frac{1}{m} \|{\cal O}_\Omega^*(\zeta)\|_2$ is taken from \cite[Proposition 4]{DQi2017}.

\begin{proposition}\label{prop:estimator_R_Omega_xi}
	Let $\zeta={\cal O}_{\Omega}(\daleth)\in\mathbb{R}^m$  and $\daleth\in\mathbb{S}^n$ be given by \eqref{eq:def-daleth}. Suppose that there exists $C_1>1$ such that $m>C_1n\log(n)$.  Then, there exists a constant $C_2>0$ such that with probability at least $1-1/n$,
	\begin{equation}\label{eq:R_Omega_ineq}
		\frac{1}{m}\left\|{\cal O}_{\Omega}^*(\zeta) \right\|_2\le C_2(2\omega\eta+\eta^2)\sqrt{\frac{\log(2n)}{nm}},
	\end{equation}
	where $\omega=\|{\cal O}_{\Omega}(\bar{d}+\bar{s})\|_{\infty}$.
\end{proposition}

The following result on  the upper bound of $ \frac{1}{m} \left\|{\cal O}^* (\xi)\right\|_\infty$ are a direct consequence of the large derivation inequality  for sums of independent sub-gaussian/sub-exponential random variables \cite[Proposition 5.10 \& 5.16]{Vershynin10}. 
\begin{proposition}\label{prop:inv_operator_inf_norm}
	Let $\zeta={\cal O}_{\Omega}(\daleth)\in\mathbb{R}^m$  and $\daleth\in\mathbb{S}^n$ be given by \eqref{eq:def-daleth}. Then, there exists a positive constant $C_3$ such that  with probability at least $1-2/n^{2}$,
	\begin{equation}\label{eq:inv_operator_inf_norm-1}
		\frac{1}{m}\left\|  {\cal O}^*_{\Omega} (\zeta) \right\|_\infty \leq C_3(2\omega\eta+\eta^2) \frac{\log(2n^2)}{m},
	\end{equation}
	where $\omega=\|{\cal O}_{\Omega}(\bar{d}+\bar{s})\|_{\infty}$.
\end{proposition}
\begin{proof}
	From \eqref{eq:def-daleth} and the definition of $\zeta$, we know that
	\[
	\left\|{\cal O}_{\Omega}^*(\zeta) \right\|_{\infty}\leq 2\omega\eta\left\|{\cal O}_{\Omega}^*(\xi) \right\|_{\infty}+\eta^2\left\|{\cal O}_{\Omega}^*(\xi\circ \xi) \right\|_{\infty}.
	\]
	Therefore, for any given $t_1$, $t_2>0$, we  have
	\begin{equation}\label{eq:propability_ineq_0}
		{\mathbb P}\left( \left\|{\cal O}_{\Omega}^*(\zeta) \right\|_{\infty}\ge 2\omega\eta t_1+\eta^2 t_2\right)\leq {\mathbb P}\left(\left\|{\cal O}_{\Omega}^*(\xi) \right\|_{\infty}\ge t_1 \right)+{\mathbb P}\left(\left\|{\cal O}_{\Omega}^*(\xi\circ \xi) \right\|_{\infty}\ge t_2 \right).
	\end{equation}
	Denote the random matrix $Y:={\cal O}_{\Omega}^*(\xi)=\sum_{l=1}^m\xi_lX_l$ and $Z:={\cal O}_{\Omega}^*(\xi\circ \xi-{\bf 1})=\sum_{l=1}^m(\xi_l^2-1)X_l$. Then, for each $i,j\in\{1,\dots,n\}$, the $(i,j)$-th elements of $Y$ and $Z$ can be written as $Y_{ij}=\sum_{l=1}^m{\bf a}^{(ij)}_l\xi_l$ and $Z_{ij}=\sum_{l=1}^{m}{\bf a}^{ij}(\xi_l^2-1)$, where ${\bf a}^{(ij)}:=((X_1)_{ij},\cdots,(X_m)_{ij})^T\in\mathbb{R}^m$. Since $\xi_l$ is an i.i.d. copy of sub-Gaussian random variables, we know that there exist positive constants $M_1$ such that $\left\| \xi_l\right\|_{\psi_1}\leq M_1$ \cite[Section 5.2.3]{Vershynin10}. Due to ${\mathbb E}(\xi_l)=0$, we know from \cite[Proposition 5.10]{Vershynin10} that there exist positive constant $C_4$ such that for each $i,j\in\{1,\ldots,n\}$ and any given $t_1>0$, 
	\[
	{\mathbb P}\left(|Y_{ij}|\ge t_1 \right)\leq {\rm exp}\left(1-\frac{C_4 t_1^2}{M_1^2\|{\bf a}^{(ij)}\|^2}  \right),
	\]
	which implies that
	\begin{equation}\label{eq:propability-l-infty-tmp0-1}
		{\mathbb P}\left(\left\|Y \right\|_{\infty}\ge t_1 \right)\leq d_{\mathbb{S}^n_h}{\rm exp}\left(1-\frac{C_4t_1^2}{M_1^2\max \|{\bf a}^{(ij)}\|^2} \right).
	\end{equation}
	Meanwhile, since $\xi_l$ is sub-Gaussian, we know that $\xi_l^2$ is an i.i.d. copy of sub-exponential random variables, which implies that there exists positive constant $M_2$ such that $\left\| \xi_l^2\right\|_{\psi_1}\leq M_2$, $l=1,\ldots,m$ (see e.g., \cite[Section 5.2.4]{Vershynin10}). Moreover, since ${\mathbb E}(\xi_l^2)=1$, we know from \cite[Proposition 5.16]{Vershynin10} that there exist positive constants $C_5$ such that for each $i,j\in\{1,\dots,n\}$ and any given $t_3>0$, 
	\[
	{\mathbb P}\left(|Z_{ij}|\ge t_3 \right)\leq 2{\rm exp}\left(-C_5\min\left\{\frac{t_3^2}{M_2^2\|{\bf a}^{(ij)}\|^2},\frac{t_3}{M_2\|{\bf a}^{(ij)}\|_{\infty}} \right\} \right),
	\]
	which implies that
	\begin{equation}\label{eq:propability-l-infty-tmp0-2}
		{\mathbb P}\left(\left\|Z \right\|_{\infty}\ge t_3 \right)\leq 2d_{\mathbb{S}^n_h}{\rm exp}\left(-C_5\min\left\{\frac{t_3^2}{M_2^2\max \|{\bf a}^{(ij)}\|^2},\frac{t_3}{M_2\max \|{\bf a}^{(ij)}\|_{\infty}} \right\} \right).
	\end{equation}
	Moreover, for each $i,j\in\{1,\dots,n\}$, it is clear that $\|{\bf a}^{(ij)}\|^2\le m_{\max}/4$ and $\|{\bf a}^{(ij)}\|_{\infty}\le 1/2$, where $m_{\max}$ is the maximum number of repetition of $(i,j)$-th index in $\Omega$. Thus, it follows from \eqref{eq:propability-l-infty-tmp0-1} and \eqref{eq:propability-l-infty-tmp0-2} that for any given $t_1$, $t_3>0$,
	\begin{equation}\label{eq:propability-l-infty-tmp-1}
		{\mathbb P}\left(\left\|Y \right\|_{\infty}\ge t_1 \right)\leq d_{\mathbb{S}^n_h}{\rm exp}\left(1-\frac{4C_4t_1^2}{M_1^2 m_{\max}}\right)
	\end{equation}
	and
	\begin{equation}\label{eq:propability-l-infty-tmp-2}
		{\mathbb P}\left(\left\|Z \right\|_{\infty}\ge t_3 \right)\leq 2d_{\mathbb{S}^n_h}{\rm exp}\left(-C_5\min\left\{\frac{4t_3^2}{M_2^2m_{\max}},\frac{2t_3}{M_2} \right\} \right).
	\end{equation}
	Therefore, it follows from Lemma \ref{lem:max_num_repetition} that with probability at least $1-1/(2n^2)$, there exists a constant $C_6>0$ such that $m_{\max}\le C_6\log(2n^2)$. Thus, by \eqref{eq:propability-l-infty-tmp-1}, we know that for any $t_1>0$,
	\begin{equation}\label{eq:propability-l-infty-1}
		{\mathbb P}\left(\left\|{\cal O}_{\Omega}^*(\xi) \right\|_{\infty}\ge t_1 \right)\leq d_{\mathbb{S}^n_h}{\rm exp}\left(1-\frac{4C_4t_1^2}{M_1^2C_6\log(2n^2)} \right)+\frac{1}{2n^2}.
	\end{equation}
	On the other hand, since  $\left\|Z \right\|_{\infty}\ge \|{\cal O}_{\Omega}^*(\xi\circ \xi)\|_{\infty}-\|{\cal O}_{\Omega}^*({\bf 1})\|_{\infty}\ge \|{\cal O}_{\Omega}^*(\xi\circ \xi)\|_{\infty}-m_{\max}/2\ge \|{\cal O}_{\Omega}^*(\xi\circ \xi)\|_{\infty}-C_6\log(2n^2)/2$ if $m_{\max}\le C_6\log(2n^2)$, we know from Lemma \ref{lem:max_num_repetition} that for any $t_2>C_6\log(2n^2)/2$, 
	\begin{eqnarray}
		{\mathbb P}\left(\|{\cal O}_{\Omega}^*(\xi\circ \xi)\|_{\infty}\ge t_2 \right)&\le& {\mathbb P}\left(\|Z\|_{\infty}\ge t_2-\frac{C_6\log(2n^2)}{2} \right)+\frac{1}{2n^2} \nonumber \\ [3pt]
		&\leq&  2d_{\mathbb{S}^n_h}{\rm exp}\left(-C_5\min\left\{\frac{(2t_2-C_6\log(2n^2))^2}{4M_2^2C_6\log(2n^2)},\frac{2t_2-C_6\log(2n^2)}{M_2} \right\} \right)+\frac{1}{2n^2}.\label{eq:propability-l-infty-2}
	\end{eqnarray}
	Therefore, by setting $t_1:=M_1\sqrt{\frac{C_6}{2C_4}}\log(2n^2)$, we obtain from \eqref{eq:propability-l-infty-1} that 
	\[
	{\mathbb P}\left(\left\|{\cal O}_{\Omega}^*(\xi) \right\|_{\infty}\ge M_1\sqrt{\frac{C_6}{2C_4}}\log(2n^2) \right)\leq \frac{1}{2n^2}+\frac{1}{2n^2}=\frac{1}{n^2}.
	\]
	Meanwhile, by setting $t_2:=\frac{4M_2+1}{2}C_6\log(2n^2)>C_6\log(2n^2)/2$, we conclude from \eqref{eq:propability-l-infty-2} that
	\[
	{\mathbb P}\left(\|{\cal O}_{\Omega}^*(\xi\circ \xi)\|_{\infty}\ge \frac{4M_2+1}{2}C_6\log(2n^2) \right)\le \frac{1}{2n^2}+\frac{1}{2n^2}=\frac{1}{n^2}.
	\]
	Finally, it follows from  \eqref{eq:propability_ineq_0} that there exists a constant $C_3>0$ such that
	\[
	{\mathbb P}\left( \left\|{\cal O}_{\Omega}^*(\zeta) \right\|_{\infty}\ge C_3(2\omega\eta+\eta^2)\log(2n^2)\right)\leq \frac{2}{n^2},
	\]
	which implies \eqref{eq:inv_operator_inf_norm-1} holds with probability at least $1-2/n^2$. This completes the proof. \hfill $\Box$ 
\end{proof}

Next, we shall present our statistical error bound results on the proposed convex model \eqref{pr:cvxopt}. Proposition \ref{prop:estimator_R_Omega_xi} and \ref{prop:inv_operator_inf_norm} suggest that the penalized parameters $\rho_D$ and $\rho_S$ based on \eqref{eq:PenalizedPara} can take the following particular values:
\begin{equation}\label{eq:parameter-rho}
	\rho_D=O\left((2\omega\eta+\eta^2)\sqrt{\frac{\log(2n)}{mn}}\right) \quad {\rm and} \quad \rho_S=O\left((2\omega\eta+\eta^2)\frac{\log(2n^2)}{m}\right),
\end{equation} 
where $\omega=\|{\cal O}_{\Omega}(\bar{d}+\bar{s})\|_{\infty}$. Moreover, it follows from \cite[(31)]{DQi2017} and \cite[Lemma 5.6]{Wu2014} that if there exists $C_1>1$ such that $m>C_1n\log(n)$, then there exist positive constants $C_4$ and $C_5$ such that $\vartheta_D$ and $\vartheta_S$ defined by \eqref{def:exp_inv_operator} satisfy 
\begin{eqnarray}\label{eq:exp_inv_operator-bound}
	\vartheta_D=\mathbb{E}\left\|\frac{1}{m}\mathcal{O}_\Omega^*(\epsilon)\right\|_2\le C_4\sqrt{\frac{\log(2n)}{mn}} \quad  {\rm and} 
	\quad \vartheta_S=\mathbb{E}\left\|\frac{1}{m}\mathcal{O}_\Omega^*(\epsilon)\right\|_\infty\le C_5\frac{\log(2n^2)}{m}.
\end{eqnarray}

Finally, by combining Proposition \ref{prop:bound_sum_norm}, \ref{prop:estimator_R_Omega_xi} and \ref{prop:inv_operator_inf_norm}, we obtain the following error bound, immediately. We omit the detail proof for the sake of brevity. 

%

\begin{theorem}\label{thm:error_bound}
	Let $(\widehat{D}, \widehat{S})$ and $(\overline{D}, \overline{S})$ be an optimal solution of \eqref{pr:cvxopt} and the underground true EDM and outlier matrices, respectively. Assume the sample size $m$ satisfies $m>C_1n\log(2n)$ for some constant $C_1>0$. For any given $\kappa_D > 1$ and $\kappa_S>1$, suppose that the  parameters $\rho_D>0$ and $\rho_S>0$ in the objective function \eqref{pr:cvxopt} satisfy \eqref{eq:parameter-rho}. Under Assumption \ref{as:prob_lower_bound}, there exist some positive constants $C_2$, $C_3$, $C_1'$, $C_2'$ and $C_3'$ such that either $\|\widehat{D}-\overline{D}\|^2+\|\widehat{S}-\overline{S}\|^2\leq \Gamma_1$
	or
	$$
	\|\widehat{D}-\overline{D}\|^2+\|\widehat{S}-\overline{S}\|^2\leq \Gamma_2
	$$
	with probability at least $1-2/n-2/n^2$, where $\Gamma_1$ and $\Gamma_2$ are defined by
	\begin{equation}\label{eq:def-Gamma1}
		\Gamma_1:=C_2\mu_1(b_D+b_S)^2d_{\mathbb{S}^n_h}\sqrt{\frac{\log(2n)}{m}}
	\end{equation}
	and
	\begin{eqnarray}
		\Gamma_2&:=& C_3\mu_1^2d_{\mathbb{S}^n_h}\left\{ C_1' \eta^2(2\omega+\eta)^2\left[ (\kappa_Da_D+2\sqrt{2})^2\frac{rd_{\mathbb{S}^n_h}\log(2n)}{nm} +\big( a_S+\frac{1}{\kappa_S}\big)^2\frac{kd_{\mathbb{S}^n_h}\log^2(2n^2)}{m^2} \right] \right. \nonumber \\ [3pt]
		& &\left. + C_2'(b_D+b_S)^2\big(\frac{\kappa_D}{\kappa_D-1}\big)^2 \left[(a_D +2\sqrt{2})^2\frac{rd_{\mathbb{S}^n_h}\log(2n)}{nm}+\big(\frac{\kappa_Sa_S+1}{\kappa_D}\big)^2\frac{kd_{\mathbb{S}^n_h}\log^2(2n^2)}{m^2} \right] \right. \nonumber \\ [3pt]
		& &\left. +C_3'(b_D+b_S)^2 \big(\frac{\kappa_S}{\kappa_S-1}\big)^2 \left[\big(\frac{\kappa_Da_D+2\sqrt{2}}{\kappa_S}\big)^2\frac{rd_{\mathbb{S}^n_h}\log(2n)}{nm}+(a_S +1)^2\frac{kd_{\mathbb{S}^n_h}\log^2(2n^2)}{m^2} \right] \right\}, \label{eq:def-Gamma2}
	\end{eqnarray}
	with $a_D$ and $a_S$ are given by \eqref{def:apprx_distance}.
\end{theorem}

We know from Theorem \ref{thm:error_bound} that since the unknown true EDM $\overline{D}$ and outlier matrix $\overline{S}$ are bounded, in order to control the estimation error, we only need samples with the size $m$ of the order $\max\{r,k\}(n-1)\log(2n)/2$, since $d_{\mathbb{S}^n_h}=n(n-1)/2$. Note that, it is reasonable to assume the embedding dimension $r= {\rm rank}(J\overline{D}J)$ and the outliers number $k$ are small. Therefore, the sample size $m$ is much smaller than $n(n-1)/2$, the total number of the off-diagonal entries. However, we shall mention that  one cannot obtain exact recovery from the bound obtained in Theorem \ref{thm:error_bound} even without noise, i.e., $\eta=0$. Furthermore, as mentioned in \cite{NWainwright12}, even for the outlier-free case (i.e., $\overline{S}\equiv 0$), this phenomenon is unavoidable due to lack of identifiability. For instance, consider the EDM $\overline{D}$ and the perturbed EDM $\widetilde{D}=\overline{D}+\varepsilon {\bf e}_1 {\bf e}_1^T$. Thus, with high probability, ${\cal O}(D^*)={\cal O}(\widetilde{D})$, which implies that it is impossible to distinguish two EDMs even if they are noiseless. If one is interested only in exact recovery in the noiseless setting, some addition assumptions such as the matrix incoherence condition (see e.g., \cite[A0]{CR09}) are necessary. In fact, recently, under matrix incoherence, random signs of outliers (i.e., the signs of the nonzero entries of $\overline{S}$ are i.i.d. symmetric
Bernoulli random variables) and other assumptions, Chen et al. \cite{CFMYan2020} obtained a near-optimal statistical guarantee of the  convex nuclear norm plus $l_1$-norm penalized model for the (unconstrained) Robust PCA  by building up the connection between the convex estimations and an auxiliary nonconvex optimization algorithm. For the Gaussian noise and squared matrices case,  the estimation error bound achievable by their estimator \cite{CFMYan2020} reads as 
		\[
		\|\widehat{D}-\overline{D}\|\le C\eta\sqrt{\frac{n}{m}} 
		\]
		with high probability, where $C>0$ is a constant. Clearly, the resulting bound is stronger than ours for the case of the (unconstrained)  Robust PCA. However, as we mentioned before,  the results obtained in  \cite{CFMYan2020} have become inadequate since  the model studied in their paper has no ``hard-constraints", e.g., the noisy correlation matrix recovery (i.e., a positive semidefinite matrix whose diagonal elements are all ones) and the EDM estimation considered in this paper. Furthermore, neither matrix incoherence nor the  random signs of outliers condition  is assumed in this paper.


\section{Recovery of the embedding dimensionality and outlier detection}\label{section:recovery-dimension-outlier}

In order to study the recovery guarantee of the dimensionality of embedding and outliers cardinality, we first introduce some useful notations and results on the proposed convex model \eqref{pr:cvxopt} in Section \ref{section:Model}. First, it is clear that the following generalized Slater condition for \eqref{pr:cvxopt} always holds:  
\begin{definition}\label{def:Slater-c}
	There exists $D^0\in\mathbb{S}^n$ and $S^0\in \mathbb{S}^n$ such that 
	\[
	{\rm diag}(D^0)=0, \quad -D^0\in {\rm int}(\mathbb{K}^n_+) \quad {\rm and} \quad S^0>0,
	\]
	where $\mathbb{K}^n_+$ is the almost positive semidefinite matrix cone defined by \eqref{eq:def_alPSD} and ${\rm int}(\mathbb{K}^n_+)$ is its interior. 
\end{definition}

%

Let $(\mathbb{K}^n_+)^{\circ}\subseteq \mathbb{S}^n$ be the polar cone of the almost positive semidefinite matrix cone $\mathbb{K}^n_+$, i.e.,
\begin{equation}\label{eq:def-polar-cone-K}
	(\mathbb{K}^n_+)^{\circ}:=\{ Z\in\mathbb{S}^n\mid \langle Z,Y\rangle \le 0\quad \forall\, Y\in \mathbb{K}^n_+  \}.
\end{equation} 
We use $H\in\mathbb{S}^n$ to denote the  Householder matrix, i.e.,
\begin{equation}\label{eq:def-H-matrix}
	H:=I-\frac{2}{{\bf u}^T{\bf u}}{\bf u}{\bf u}^T \quad {\rm with} \quad {\bf u}:=(1,\dots,1,\sqrt{n}+1)^T\in\mathbb{R}^n.
\end{equation}
It is clear that the  Householder matrix $H$ is symmetric and orthogonal (i.e., $H^2=I$). Also, the centering matrix $J$ defined by \eqref{eq:def-J} satisfies
\begin{equation}\label{eq:J-eq}
	J=H\left[\begin{array}{cc}
		I_{n-1} & 0 \\ [3pt]
		0 & 0
	\end{array}\right]H.
\end{equation}  
For any $X\in\mathbb{S}^n$, we rewrite the matrix $HXH$ as the following block form: 
\begin{equation}\label{eq:HXH-block}
	HXH=\left[\begin{array}{cc}
		\widetilde{X}_{11} & \tilde{x} \\ [3pt]
		\tilde{x}^T & \tilde{x}_0
	\end{array}\right] \quad {\rm with} \quad \widetilde{X}_{11}\in\mathbb{S}^{n-1}, \quad \tilde{x}\in\mathbb{R}^{n-1} \quad {\rm and} \quad \tilde{x}_0\in\mathbb{R}.
\end{equation}
Moreover, by \eqref{eq:J-eq} and simple calculations, we obtain the following basic identity: 
\begin{equation}\label{eq:JXJ-form}
	JXJ= H\left[\begin{array}{cc}
		\widetilde{X}_{11} & 0 \\ [3pt]
		0 & 0
	\end{array}\right] H,
\end{equation}
where $\widetilde{X}_{11}\in\mathbb{S}^{n-1}$ is the first block defined by \eqref{eq:HXH-block} for $HXH$. 

By \cite[Theorem 2.1]{HWells88}, we have the following characterizations on $\mathbb{K}^n_+$ and its polar $(\mathbb{K}^n_+)^{\circ}$:
\begin{equation}\label{eq:K_+}
	\mathbb{K}^n_+=\left\{ H\left[\begin{array}{cc}
		Z & z \\ [3pt]
		z^T & z_0
	\end{array}\right]H\in\mathbb{S}^n \mid Z\in \mathbb{S}^{n-1}_+, \ z\in\mathbb{R}^{n-1}, \ z_0\in\mathbb{R} \right\}
\end{equation}
and 
\begin{equation}\label{eq:K_+-polar}
	(\mathbb{K}^n_+)^{\circ}=\left\{ H\left[\begin{array}{cc}
		Z & 0 \\ [3pt]
		0 & 0
	\end{array}\right]H\in\mathbb{S}^n \mid Z\in \mathbb{S}^{n-1}_- \right\}.
\end{equation}
Thus, for any given integer $1\le r\le n$,  by \eqref{eq:JXJ-form} and \eqref{eq:K_+}, we know that $X\in\mathbb{K}_+^n$ and ${\rm rank}(JXJ)\le r$ if and only if $\widetilde{X}_{11}\in \mathbb{S}_+^{n-1}$ and ${\rm rank}(\widetilde{X}_{11})\le r$, and $\mathbb{K}^n_+\ni X \perp Y\in (\mathbb{K}^n_+)^{\circ}$ if and only if
\begin{equation}\label{eq:comp-SDP}
	\mathbb{S}^{n-1}_+\ni\widetilde{X}_{11} \perp \widetilde{Y}_{11}\in \mathbb{S}^{n-1}_- \quad {\rm and} \quad HYH= \left[\begin{array}{cc}
		\widetilde{Y}_{11} & 0 \\ [3pt]
		0 & 0
	\end{array}\right],
\end{equation} 
where $\widetilde{X}_{11}\in\mathbb{S}^{n-1}$ and $\widetilde{Y}_{11}\in\mathbb{S}^{n-1}$ are the first blocks defined by \eqref{eq:HXH-block} for $HXH$ and $HYH$, respectively (see also \cite[Lemma 2.1]{QiYuan14} for  details).

Let $\widetilde{D}\in\mathbb{S}^n$ and $\widetilde{S}\in\mathbb{S}^n$ be the given initial estimators.  Recall that $\widetilde{F}\in \mathbb{S}^n$ and $\widetilde{G}\in \mathbb{S}^n$ are the symmetric matrices defined in \eqref{eq:def-F} and \eqref{eq:def-G} with respect to $(\widetilde{D}, \widetilde{S})$. For the given $(\widetilde{D},\widetilde{S})$, denote $\tilde{t}_D\in\mathbb{R}$ and $\tilde{t}_S\in\mathbb{R}$ by
\begin{equation}\label{eq:def-tDtS}
	\tilde{t}_D=\left\{\begin{array}{ll}
		\frac{\lambda_{r+1}(-\widetilde{D}_{11})}{\lambda_{1}(-\widetilde{D}_{11})} & \mbox{if $\widetilde{D}\neq 0$,} \\ [3pt]
		0 & \mbox{otherwise}
	\end{array}\right. \quad {\rm and} \quad \tilde{t}_S=\left\{\begin{array}{ll}
		\frac{\widetilde{S}_{i'j'}}{\max_{k,l}\{\widetilde{S}_{kl}\}} & \mbox{if $\widetilde{S}\neq 0$,} \\ [3pt]
		0 & \mbox{otherwise},
	\end{array}\right.
\end{equation}
where $(i',j')\in\{1,\ldots,n\}\times \{1,\ldots,n\}$ be the index such that $\widetilde{S}_{i'j'}=\max\left\{\widetilde{S}_{ij} \mid (i,j)\notin{\rm supp}(\overline{S})\right\}$.
Now, we are ready to present the results on the guarantee of recovery of the embedding dimensionality and outlier detection.

\begin{theorem}\label{thm:rank-sparse-guarantee}
	Let $(\widehat{D}, \widehat{S})$ and $(\overline{D}, \overline{S})$ be an optimal solution of \eqref{pr:cvxopt} and the underground true EDM and outlier matrices, respectively. Assume the sample size satisfies $m>C_0n\log(2n)$ for some constant $C_0>0$. Suppose that the initial estimators $\widetilde{D}$ and $\widetilde{S}$ satisfy $\tilde{t}_D\in[0,1)$ and $\tilde{t}_S\in[0,1)$, where $\tilde{t}_D$ and $\tilde{t}_S$ are defined by \eqref{eq:def-tDtS}.  Let $\widetilde{F}$ and $\widetilde{G}$ be   the symmetric matrices defined by \eqref{eq:def-F} and \eqref{eq:def-G} with respect to $(\widetilde{D}, \widetilde{S})$. Suppose that the parameters $\rho_D>0$ and $\rho_S>0$ in the objective function \eqref{pr:cvxopt} defined by \eqref{eq:parameter-rho} satisfying $\rho_D>\frac{\varepsilon^\tau+\tilde{t}_D^\tau}{\varepsilon^{\tau}(1-\tilde{t}_D^{\tau})}C(2\omega\eta+\eta^2)\sqrt{\frac{\log(2n)}{mn}}$ and $\rho_S>\frac{\varepsilon^\tau+\tilde{t}_S^\tau}{\varepsilon^{\tau}(1-\tilde{t}_S^{\tau})}C(2\omega\eta+\eta^2)\frac{\log(2n^2)}{m}$ for some large constant $C>0$. Then, we have 
	\[
	{\rm rank}(-J\widehat{D}J)\le {\rm rank}(-J\overline{D}J)\quad {\rm and} \quad {\rm supp}(\widehat{S})\subseteq {\rm supp}(\overline{S}) 
	\]
	with probability at least $1-1/n-3/n^2$. Furthermore, in addition, if $\|\widehat{D}-\overline{D}\|<\lambda_r(-J\overline{D}J)$ and $\|\widehat{S}-\overline{S}\|<\min\left\{\overline{S}_{ij} \mid (i,j)\in{\rm supp}(\overline{S})\right\}$, then with the same probability, we have
	\[
	{\rm rank}(-J\widehat{D}J)= {\rm rank}(-J\overline{D}J)\quad {\rm and} \quad {\rm supp}(\widehat{S})= {\rm supp}(\overline{S}). 
	\] 
\end{theorem}
\begin{proof}
	Since \eqref{pr:cvxopt} is convex and the generalized Slater condition (Definition \ref{def:Slater-c}) always holds, we know that there exist Lagrangian multipliers $(z,\Gamma,U)\in\mathbb{R}^n\times\mathbb{S}^n\times\mathbb{S}^n$ such that $(\widehat{D},\widehat{S})$ satisfies the following Karush-Kuhn-Tucker (KKT) condition:
	\begin{equation}\label{eq:KKT-cvxopt}
		\left\{ \begin{array}{l}
			-\frac{1}{m}\mathcal{O}_{\Omega}^*\left({\bf y}-\mathcal{O}_{\Omega}(\widehat{D}+\widehat{S})\right)-\rho_DJ(I-\widetilde{F})J-{\rm Diag}(z)-\Gamma=0, \\ [5pt]
			-\frac{1}{m}\mathcal{O}_{\Omega}^*\left({\bf y}-\mathcal{O}_{\Omega}(\widehat{D}+\widehat{S})\right)+\rho_S(E-\widetilde{G})+U=0, \\ [5pt]
			{\rm diag}(\widehat{D})=0, \\ [5pt]
			\mathbb{K}^n_+ \ni -\widehat{D} \perp \Gamma\in (\mathbb{K}^n_+)^{\circ}, \quad 0\le \widehat{S}_{ij}\perp U_{ij}\le 0,\quad i,j\in\{1,\ldots,n\}.
		\end{array}
		\right.
	\end{equation}
	
	Consider the first equation of \eqref{eq:KKT-cvxopt}. By denoting $\Upsilon:=-\frac{1}{m}\mathcal{O}_{\Omega}^*\left({\bf y}-\mathcal{O}_{\Omega}(\widehat{D}+\widehat{S})\right)$, we obtain that 
	\begin{equation}\label{eq:first-eq-KKT}
		H{\rm Diag}(z)H=H\Upsilon H-\rho_DH(J(I-\widetilde{F})J)H-H\Gamma H,
	\end{equation}
	where $H$ is the Householder matrix defined by \eqref{eq:def-H-matrix}.
	Since $\Gamma\in (\mathbb{K}^n_+)^{\circ}$, we know from \eqref{eq:JXJ-form} and \eqref{eq:K_+-polar} that the last columns of  the symmetric matrices $H(J(I-\widetilde{F})J)H$ and $H\Gamma H$ are all zero. Moreover, for any $z\in\mathbb{R}^n$, we know that the last column of $H{\rm Diag}(z)H$ can be calculated as follows
	\[
	(H{\rm Diag}(z)H)_{(:,n)}=-\frac{1}{\sqrt{n}}Hz\in\mathbb{R}^n.
	\]   
	Consequently, we know from \eqref{eq:first-eq-KKT} that the multiplier $z$ can be characterized by
	\[
	z=-\sqrt{n}H\left[\begin{array}{c}
		\tilde{\upsilon} \\ [3pt]
		\tilde{\upsilon}_0
	\end{array}\right],
	\]
	where $(\tilde{\upsilon},\tilde{\upsilon}_0)^T\in\mathbb{R}^n$ is the last column of $H\Upsilon H$ in the form \eqref{eq:HXH-block}. Thus, since $\Upsilon\in\mathbb{S}^n_h$, we know from Lemma \ref{lm:rank} that 
	\[
	\|{\rm Diag}(z)\|_2\le \|{\rm Diag}(z)\| \le 2\sqrt{n}\frac{1}{\sqrt{n}} \|\Upsilon-J\Upsilon J\|\le  2\sqrt{2} \|\Upsilon-J\Upsilon J\|_2\le 4\sqrt{2} \|\Upsilon\|_2.
	\]
	Meanwhile, by \eqref{eq:def-ob-y} and \eqref{eq:def-D-tilde}, we have
	\[
	\Upsilon=\frac{1}{m}\mathcal{O}_{\Omega}^*\left(\mathcal{O}_{\Omega}(\widehat{D}+\widehat{S})-{\bf y}\right)=\frac{1}{m}\mathcal{O}_{\Omega}^*\left(\mathcal{O}_{\Omega}(A_D+A_S)-\zeta\right),
	\]
	where $\zeta={\cal O}_{\Omega}(\daleth)$, $A_D=\widehat{D}-\overline{D}$ and $A_S=\widehat{S}-\overline{S}$. Thus, we know that 
	\[
	\|\Upsilon\|_2\le \frac{1}{m}\left\|{\cal O}_{\Omega}^*{\cal O}_{\Omega}\right\|_2\left\|A_D+A_S\right\|+\frac{1}{m}\|{\cal O}_{\Omega}^*(\zeta)\|_2.
	\]
	Under Assumption \ref{ass:level set}, we know that there exists a constant $C_1>0$ such that $\left\|A_D+A_S\right\|\le C_1$. Moreover, by combining with Lemma \ref{lem:max_num_repetition}, Proposition \ref{prop:estimator_R_Omega_xi} and $m>C_0n\log(2n)$, we obtain that there exist positive constants $C_2>0$ and $C_3>0$ such that 
	\[
	\|\Upsilon\|_2\le C_1C_2\frac{\log(2n^2)}{m} + C_3(2\omega\eta+\eta^2)\sqrt{\frac{\log(2n)}{nm}}\le C_4(2\omega\eta+\eta^2)\sqrt{\frac{\log(2n)}{nm}},
	\]
	with probability at least $1-1/(2n^2)-1/n$, where $C_4:=C_1C_2+C_3>0$. Thus, by the assumption, we know that 
	\begin{equation}\label{eq:rhoD-Gamma}
		\rho_D>(1+4\sqrt{2})\frac{\varepsilon^\tau+\tilde{t}_D^\tau}{\varepsilon^{\tau}(1-\tilde{t}_D^{\tau})}\|\Upsilon\|_2
	\end{equation} 
	with probability at least $1-1/(2n^2)-1/n$.
	
	Denote $r={\rm rank}(-J\overline{D}J)$. It is clear that $r\le n-1$.
	By \eqref{eq:def-F}, we know that $\widetilde{F}$ satisfies
	\[
	J\widetilde{F}J=H\left[\begin{array}{cc}
		\widetilde{\Lambda} & 0 \\ [3pt]
		0 & 0
	\end{array}\right]H\quad {\rm with}\quad \widetilde{\Lambda}=Q
	{\rm Diag}(f(\lambda(-\widetilde{D}_{11}))Q^T\in\mathbb{S}^{n-1},
	\]
	where $f:\mathbb{R}^{n-1}\to\mathbb{R}^{n-1}$ is the symmetric function defined by \eqref{eq:def-sym-f}, and $Q$ is a given $(n-1)\times (n-1)$ orthogonal matrix. It then follows from the well-known Weyl eigenvalue inequality \cite{Weyl12} (see also \cite[Theorem 4.3.7]{HJohnson85}) that
	\begin{eqnarray}
		\lambda_{r+1}(\widetilde{\Gamma}_{11})&\le& \lambda_1(H(\Upsilon-{\rm Diag}(z))H) + \lambda_{r+1}(-\rho_D(I_{n-1}-\widetilde{\Lambda})) \nonumber \\ [3pt]
		&\le& \|\Upsilon-{\rm Diag}(z)\|_2+\rho_D\lambda_{r+1}(\widetilde{\Lambda})-\rho_D \nonumber \\ [3pt]
		&\le& (1+4\sqrt{2})\|\Upsilon\|_2+\rho_D\left(\phi(\tilde{t}_D)-1\right), \label{eq:lambda-ineq}
	\end{eqnarray}
	where $\phi:\mathbb{R}\to\mathbb{R}$ is the scalar function given by \eqref{eq:def-phi-scalar}, which implies that
	\[
	\lambda_{r+1}(\widetilde{\Gamma}_{11})\le (1+4\sqrt{2})\|\Upsilon\|_2+\rho_D\frac{\varepsilon^{\tau}(\tilde{t}_D^{\tau}-1)}{\tilde{t}_D^{\tau}+\varepsilon^{\tau}}.
	\]
	Thus, we know from \eqref{eq:rhoD-Gamma} that with probability at least $1-1/(2n^2)-1/n$, $\lambda_{r+1}(\widetilde{\Gamma}_{11})<0$. Since $\mathbb{K}^n_+ \ni -\widehat{D} \perp \Gamma\in (\mathbb{K}^n_+)^{\circ}$, we know from \eqref{eq:comp-SDP} that 
	\begin{equation}\label{eq:rank-ineq}
		{\rm rank}(-J\widehat{D}J)\le {\rm rank}(-J\overline{D}J) 
	\end{equation} 
	with probability at least $1-1/(2n^2)-1/n$.
	
	Meanwhile, we know from the second equation of \eqref{eq:KKT-cvxopt} that the multiplier $U\in \mathbb{S}^n$ is given by
	\[
	U=-\Upsilon-\rho_S(E-\widetilde{G}).
	\]
	Again, under Assumption \ref{ass:level set}, we know that there exists a constant $C_1>0$ such that $\left\|A_D+A_S\right\|\le C_1$. Moreover, by combining with Lemma \ref{lem:max_num_repetition} and Proposition \ref{prop:inv_operator_inf_norm}, we obtain that there exist positive constants $C_2$ and $C_3$ such that 
	\[
	\|\Upsilon\|_{\infty}\le C_1C_2\frac{\log(2n^2)}{m} + C_3\frac{\log(2n^2)}{m}\le C_4\frac{\log(2n^2)}{m},
	\]
	with probability at least $1-5/(2n^2)$, where $C_4:=C_1C_2+C_3>0$. Therefore, by the assumption, we have
	\begin{equation}\label{eq:rhoS-Gamma}
		\rho_S>\frac{\varepsilon^\tau+\tilde{t}_S^\tau}{\varepsilon^{\tau}(1-\tilde{t}_S^{\tau})}\|\Upsilon\|_{\infty}
	\end{equation} 
	with probability at least $1-5/(2n^2)$.
	
	Let $(\hat{i},\hat{j})\in\{1,\ldots,n\}\times \{1,\ldots,n\}$ be the index such that $\widehat{S}_{\hat{i}\hat{j}}=\max\left\{\widehat{S}_{ij} \mid (i,j)\notin{\rm supp}(\overline{S})\right\}$. By \eqref{eq:def-G}, we know that 
	\[
	U_{\hat{i}\hat{j}} =-\Upsilon_{\hat{i}\hat{j}}+\rho_S\left(\phi(\widetilde{S}_{\hat{i}\hat{j}}/\max_{k,l}\{\widetilde{S}_{kl}\})-1\right),
	\]
	where $\phi:\mathbb{R}\to\mathbb{R}$ is the scalar function given by \eqref{eq:def-phi-scalar}. It is clear from \eqref{eq:def-tDtS} that $\widetilde{S}_{\hat{i}\hat{j}}/\max_{k,l}\{\widetilde{S}_{kl}\}\le \tilde{t}_S$. Therefore, since $\phi$ is non-decreasing, we have
	\begin{equation}\label{eq:U-ineq}
		U_{\hat{i}\hat{j}} =-\Upsilon_{\hat{i}\hat{j}}+\rho_S\left(\phi(\widetilde{S}_{\hat{i}\hat{j}}/\max_{k,l}\{\widetilde{S}_{kl}\})-1\right)\le -\Upsilon_{\hat{i}\hat{j}}+\rho_S\left(\phi(\tilde{t}_S)-1\right).
	\end{equation}
	Thus, we know from \eqref{eq:U-ineq} and \eqref{eq:def-phi-scalar} that
	\[
	U_{\hat{i}\hat{j}}\le \|\Upsilon\|_{\infty}+\rho_S(\phi(\tilde{t}_S)-1) =\|\Upsilon\|_{\infty}+\rho_S\frac{\varepsilon^{\tau}(\tilde{t}_S^{\tau}-1)}{\tilde{t}_S^{\tau}+\varepsilon^{\tau}}.
	\]
	This, together with \eqref{eq:rhoS-Gamma}, yields $U_{\hat{i}\hat{j}}<0$ with probability at least $1-5/(2n^2)$. 
	Moreover, since $0\le \widehat{S}_{ij}\perp U_{ij}\le 0$ for any $i,j\in\{1,\ldots,n\}$, we know that with probability at least $1-5/(2n^2)$, $\widehat{S}_{\hat{i}\hat{j}}=0$. By noting that for any $(i,j)\notin{\rm supp}(\overline{S})$, $0\le \widehat{S}_{ij}\le \widehat{S}_{\hat{i}\hat{j}}=0$, we conclude that 
	\begin{equation}\label{eq:supp-inclu}
		{\rm supp}(\widehat{S})\subseteq {\rm supp}(\overline{S})
	\end{equation}
	with probability at least $1-5/(2n^2)$. 
	
	By combining \eqref{eq:rank-ineq} and \eqref{eq:supp-inclu}, we obtain that with probability at least $1-1/n-3/(n^2)$, ${\rm rank}(-J\widehat{D}J)\le {\rm rank}(-J\overline{D}J)$ and ${\rm supp}(\widehat{S})\subseteq {\rm supp}(\overline{S})$. This completes the  proof of the first part.

	Next, we proceed with the proof of the second part. We know from the assumption $\|\widehat{D}-\overline{D}\|< \lambda_{r}(-J\overline{D}J)$ that
	$$
	|\lambda_{r}(-J\widehat{D}J)-\lambda_{r}(-J\overline{D}J)|\le \|J\widehat{D}J-J\overline{D}J\|\le \|\widehat{D}-\overline{D}\|< \lambda_{r}(-J\overline{D}J),
	$$ 
	which implies that $\lambda_{r}(-J\widehat{D}J)>0$. This yields 
	\begin{equation}\label{eq:rank-big}
		{\rm rank}(-J\widehat{D}J)\ge {\rm rank}(-J\overline{D}J).
	\end{equation}
	Meanwhile, let $(\bar{i},\bar{j})\in\{1,\ldots,n\}\times \{1,\ldots,n\}$ be the index such that $\overline{S}_{\bar{i}\bar{j}}=\min\left\{\overline{S}_{ij} \mid (i,j)\in{\rm supp}(\overline{S})\right\}$. Again, we know from the assumption $\|\widehat{S}-\overline{S}\| < \overline{S}_{\bar{i}\bar{j}}$ that 
	$$
	|\widehat{S}_{ij}-\overline{S}_{ij}|\le \|\widehat{S}-\overline{S}\| < \overline{S}_{\bar{i}\bar{j}}\quad \forall\,(i,j)\in {\rm supp}(\overline{S}).
	$$ 
	This yields that for any $(i,j)\in {\rm supp}(\overline{S})$, $\widehat{S}_{ij}>\overline{S}_{ij}-\overline{S}_{\bar{i}\bar{j}}\ge 0$, which implies that
	\begin{equation}\label{eq:supp-big}
		{\rm supp}(\widehat{S})\supseteq {\rm supp}(\overline{S}).
	\end{equation}
	Therefore, by combining \eqref{eq:rank-big} and \eqref{eq:supp-big}, we know from the first part of this theorem that with probability at least $1-1/n-3/(n^2)$, ${\rm rank}(-J\widehat{D}J)= {\rm rank}(-J\overline{D}J)$ and ${\rm supp}(\widehat{S})= {\rm supp}(\overline{S})$. The proof is completed.  \hfill $\Box$ 
\end{proof}

\begin{remark}
	In our implementations, we may choose the initial estimators $\widetilde{D}$ and $\widetilde{S}$ obtained by the nuclear norm $l_1$-minimization EDM problem \eqref{pr:NL1EDM} to generate $\widetilde{F}$ and $\widetilde{G}$ by \eqref{eq:def-F} and \eqref{eq:def-G}, since the corresponding $\tilde{t}_D$ and $\tilde{t}_S$ satisfy $\tilde{t}_D\in[0,1)$ and $\tilde{t}_S\in[0,1)$ with high probability. 
	Moreover, by combining Theorem \ref{thm:error_bound} and Theorem \ref{thm:rank-sparse-guarantee}, we know that if in addition $\lambda_r(-J\overline{D}J)>\max\{\Gamma_1^{1/2},\Gamma_2^{1/2}\}\ge 0$ and 
	\begin{equation}\label{eq:suff-c-outlier}
		\min\left\{\overline{S}_{ij} \mid (i,j)\in{\rm supp}(\overline{S})\right\}> \max\{\Gamma_1^{1/2},\Gamma_2^{1/2}\}\ge 0,
	\end{equation} 
	where $\Gamma_1$ and $\Gamma_2$ are defined by \eqref{eq:def-Gamma1} and \eqref{eq:def-Gamma2}, respectively,  then 
	\[
	{\rm rank}(-J\widehat{D}J)= {\rm rank}(-J\overline{D}J)\quad {\rm and} \quad {\rm supp}(\widehat{S})= {\rm supp}(\overline{S}) 
	\] 
	with probability at least $1-3/n-5/n^2$.
\end{remark}

\section{Numerical experiments}\label{section:numerical}
In this section, we shall demonstrate and verify the theoretical results obtained in Section  \ref{section:error-bounds} and \ref{section:recovery-dimension-outlier} for the proposed matrix optimization  model \eqref{pr:cvxopt}  by numerical experiments. In this paper, we directly employ the symmetric Gauss-Seidel decomposition based proximal alternating direction method of multipliers ({\tt sGS-ADMM}) (cf. \cite{STYang15,LSToh19}) to solve the proposed matrix optimization model \eqref{pr:cvxopt}. The detail algorithm for solving \eqref{pr:cvxopt} can be found in \cite{DQi2017b}. The numerical examples 
were tested on Matlab (2019b) under a Windows 10 64-bit Desktop (4 core, Intel Core i7-4790K @
4.00 GHZ, 16 GB RAM). We terminate {\tt sGS-ADMM} if the KKT condition \cite[(27)]{DQi2017b}  are met, i.e.,  
\begin{equation}\label{eq:stopping}
	\max\{R_{p},\; R_{d}, \; \mbox{ \tt rel\_gap} \}\leq 10^{-4},
\end{equation}
where $R_{p}$, $R_d$ and {\tt rel\_gap} are the relative infeasibilities of the primal problem \eqref{pr:cvxopt} and its dual problem, and the relative primal-dual gap, respectively, which are given by  \cite[(28)]{DQi2017b}.

In order to demonstrate and verify the theoretical results, we only focus on the examples coming from a simulated network. For numerical performance results of the proposed  model on real-world applications such as the no-line-sight mitigation in collaborative position localization, one may refer \cite{DQi2017b} for more details.   Consider a randomly generated network in $\mathbb{R}^r$ with $r=2$, where $n$ points $\{p_i\}\in\mathbb{R}^r$ located randomly in the square area $[0, 100]\times [0, 100]$. We construct the observation operator ${\cal O}_{\Omega}$ defined in \eqref{eq:def_obser_op} by picking $\{X_1,\ldots,X_m\}$ uniformly at random from the standard basis matrices of the hellos space $\mathbb{S}^n_h$ with the sample size $m=O(rn\log(2n))$. Meanwhile, we randomly add $k$ outliers which are modeled as the i.i.d. random variables to the true pairwise distances. The i.i.d. noise errors  $\xi_{ij}$ in \eqref{eq:distance_estimation} follow a zero-mean Gaussian distribution with standard deviation and the noise magnitude control factor $\eta = 0.5$. In all numerical experiments conducted in this paper, the parameters $\rho_D>0$ and $\rho_S>0$ in the convex  model \eqref{pr:cvxopt} are chosen exactly based on the rules suggested in \eqref{eq:parameter-rho}, i.e.,  $\rho_D=O\big(\sqrt{{\log(2n)}/{mn}}\big)$ and $\rho_S=O\big({\log(2n^2)}/{m}\big)$. Meanwhile, the symmetric matrices $\widetilde{F}$ and $\widetilde{G}$ are defined by \eqref{eq:def-F} and \eqref{eq:def-G}  with respect to the initial estimators $\widetilde{D}$ and $\widetilde{S}$.  In particular, we adopt the recommendation provided in \cite[(25) and (26)]{MPSun16} and \cite[Chapter 5.3]{Wu2014} with $\varepsilon\approx 0.05$ (within $0.01\sim 0.1$), $\tau =2$ (within $1\sim 3$) for the scalar function $\phi$ defined by \eqref{eq:def-phi-scalar}. Also,  the initial estimators $\widetilde{D}$ and $\widetilde{S}$ are generated from the nuclear norm and $l_1$ penalized least squares problem (i.e., the convex problem defined in \eqref{pr:cvxopt} with $\widetilde{F}=0$ and $\widetilde{G}=0$). It seems that these particular settings for $\widetilde{F}$, $\widetilde{G}$ and initial estimators work quite well based on our numerical experiments. 

\vskip 10 true pt

\noindent{\bf Example 1}. In this example, we use a simulation network with $n=1000$ random points to demonstrate the quality of the proposed estimators for different parameters $\rho_D$ and $\rho_S$. Here, $k=457$ outliers are modeled as the i.i.d. exponential random variables with the rate parameter $\lambda=100$. The sample size $m$ of the random  observation operator ${\cal O}_{\Omega}$ equals to $15202$, which is in the order of $O(rn\log(2n))$. The numerical performance for different parameters $\rho_D$ and $\rho_S$ are illustrated in Figure \ref{fig:example-rho}. The blue lines in both Figure \ref{fig:1-sub1} and \ref{fig:1-sub2}  indicate the relative errors of both estimated EDM and outlier matrices, i.e., $\|\widehat{D}-\overline{D}\|+\|\widehat{S}-\overline{S}\|$, with respect to $\log(\rho_D)$ and $\rho_S$, respectively. It can be seen clearly when $\rho_D$ and $\rho_S$ are increasing, the relative errors are decreasing. The red line in Figure \ref{fig:1-sub1} stands for the estimated embedding dimension, i.e., ${\rm rank}(-J\widehat{D}J)$, which indicates that the estimated embedding dimension is always less or equal to the true $r=2$ when $\rho_D$ and $\rho_S$ are large enough. In fact, it actually equals $r=2$ for large $\rho_D$ and $\rho_S$. It is worth to note that in this example the $r$-th eigenvalue $\lambda_r(-J\overline{D}J)$ of $-J\overline{D}J$ is in the order of $O(10^5)$, which is much larger than $\max\{\Gamma_1^{1/2},\Gamma_2^{1/2}\}\approx 10^{3}$, where $\Gamma_1$ and $\Gamma_2$ are the error bounds defined in \eqref{eq:def-Gamma1} and \eqref{eq:def-Gamma2}, respectively. Meanwhile, the red line in Figure \ref{fig:1-sub2} represents the number of outlier detection errors obtained by the proposed convex model, which includes both false-negative errors denoted by {\tt FN} (i.e., $\widehat{S}_{ij}=0$ but the true $\overline{S}_{ij}>0$) and false-positive errors denoted by {\tt FP} (i.e., $\widehat{S}_{ij}>0$ but the true $\overline{S}_{ij}=0$). We know from Figure \ref{fig:1-sub2} and Table \ref{tab:T1} that the detected outlier number ${\tt nz}\_{\tt S}=456$, $\#{\tt FN}=1$ and $\#{\tt FP}=0$, when $\rho_D$ and $\rho_S$ are large enough. This implies that ${\rm supp}(\widehat{S})\subseteq {\rm supp}(\overline{S})$, which is consistent with Theorem \ref{thm:rank-sparse-guarantee}. Interestingly, in this example, we find that $\min\left\{\overline{S}_{ij} \mid (i,j)\in{\rm supp}(\overline{S})\right\}\approx  25.2695$ is much smaller that $\|\widehat{S}-\overline{S}\|\approx 10^{3}$.   The detail numerical performance can be found in Table \ref{tab:T1}.

\begin{figure*}
	\centering
	\begin{subfigure}{.5\textwidth}
		\centering
		\includegraphics[width=1\linewidth]{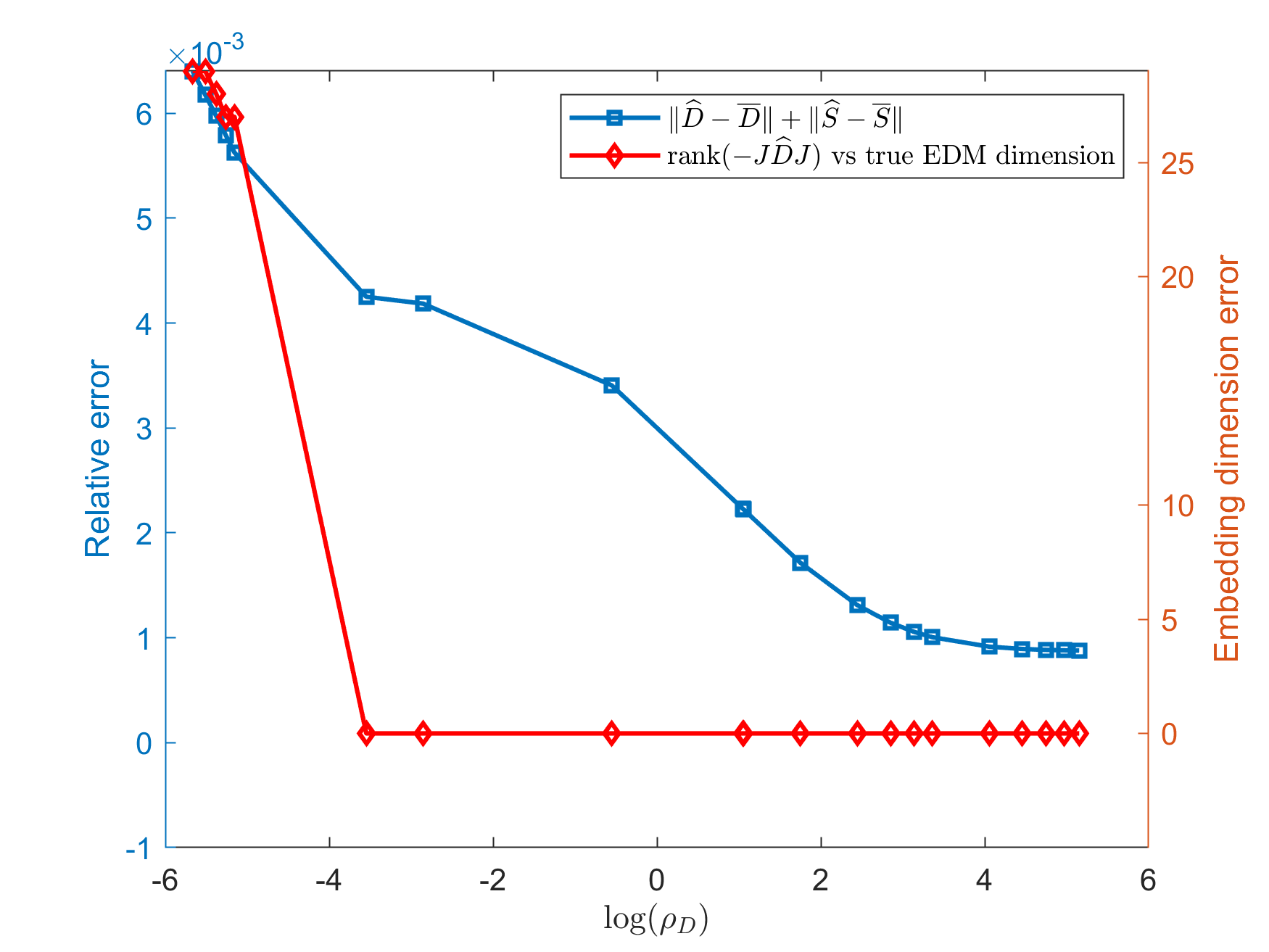}
		\caption{Relative error and dimension recovery error}
		\label{fig:1-sub1}
	\end{subfigure}%
	\begin{subfigure}{.5\textwidth}
		\centering
		\includegraphics[width=1\linewidth]{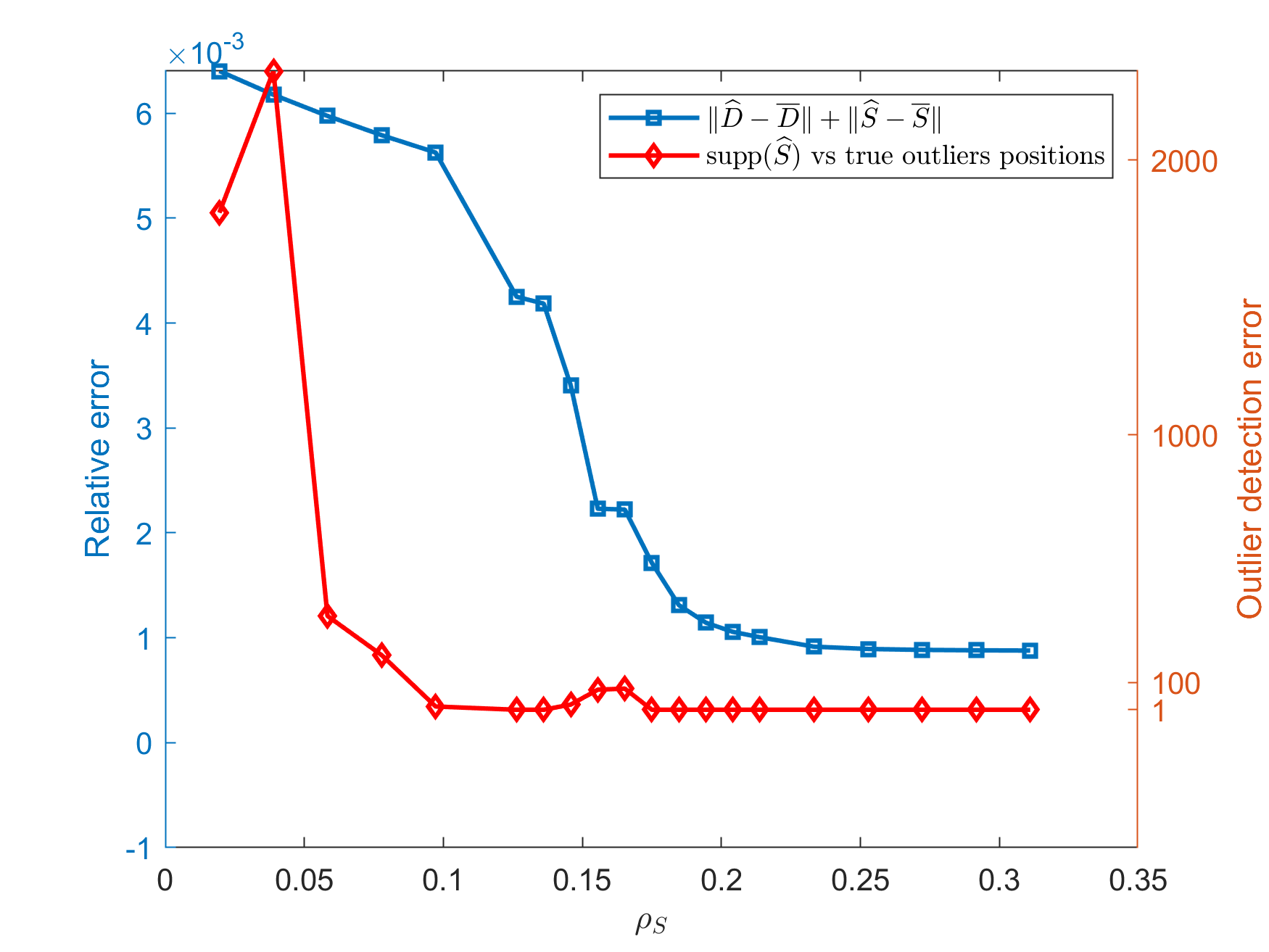}
		\caption{Relative error and outlier detection error}
		\label{fig:1-sub2}
	\end{subfigure}
	\caption{Performance comparison for different $\rho_D$ and $\rho_S$}
	\label{fig:example-rho}
\end{figure*}

\begin{table*}
	\begin{center} 
		{\footnotesize\begin{tabular}{c|cc|c}
				\hline
				($\rho_D$, $\rho_S$) & ($R_p$, $R_d$, {\tt rel\_gap}) &  ({\tt rel\_err},{\tt r\_dim},{\tt nz\_S},{\tt mis\_O}) & {\tt cpu(s)}\\ [3pt]
				\hline
				(0.0035, 0.0194)   & (4.25e-06, 7.18e-06, 8.83e-05) &  (0.0064, 31, 2265, 1808) & 360.84 \\ [3pt]
				(0.0040, 0.0389)   & (6.35e-06, 1.03e-05, 8.09e-05) &  (0.0062, 31, 2323, 2780) & 319.67 \\ [3pt]
				(0.0034, 0.0865)   & (1.61e-06, 2.35e-06, 9.55e-05) &  (0.0060, 30, 797, 342) & 321.04 \\ [3pt]
				(0.0052, 0.0778)   & (1.28e-06, 2.27e-06, 9.71e-05) &  (0.0058, 29, 655, 200) & 336.43 \\ [3pt]
				(0.0058, 0.0972)   & (5.27e-07, 5.07e-07, 9.91e-05) &  (0.0056, 29, 468, 13) & 430.75 \\ [3pt]
				(0.0288, 0.1264)   & (3.72e-07, 5.56e-07, 9.96e-05) &  (0.0043, 2, 456, 1) & 363.16 \\ [3pt]
				(0.0576, 0.1361)   & (3.17e-07, 4.35e-07, 9.94e-05) &  (0.0042, 2, 456, 1) & 304.29 \\ [3pt]
				(0.5762, 0.1458)   & (9.41e-07, 9.97e-07, 6.69e-05) &  (0.0034, 2, 475, 20) & 142.06 \\ [3pt]
				(2.8520, 0.1555)   & (4.73e-06, 8.29e-06, 2.50e-06) &  (0.0022, 2, 529, 74) & 328.50 \\ [3pt]
				(2.8808, 0.1653)   & (5.91e-06, 1.00e-05, 3.18e-06) &  (0.0022, 2, 533, 78) & 325.34 \\ [3pt]
				(5.7616, 0.1750)   & (4.71e-07, 3.85e-07, 9.59e-05) &  (0.0017, 2, 456, 1) & 795.77 \\ [3pt]
				(11.5233, 0.1847)   & (1.12e-07, 2.15e-08, 9.73e-05) &  (0.0013, 2, 456, 1) & 1357.65 \\ [3pt]
				(17.2849, 0.1944)  & (2.52e-08, 5.78e-09, 9.93e-05) &  (0.0011, 2, 456, 1) & 1890.97 \\ [3pt]
				(23.0465, 0.2041)  & (1.57e-09, 4.22e-09, 9.98e-05) &  (0.0011, 2, 456, 1) & 2401.03 \\ [3pt]
				(28.8082, 0.2139)  & (6.43e-10, 1.41e-09, 1.00e-04) &  (0.0010, 2, 456, 1) & 2783.24 \\ [3pt]
				(57.6164, 0.2333)  & (5.60e-11, 3.89e-11, 9.98e-05) &  (0.0009, 2, 456, 1) & 4079.01 \\ [3pt]
				(86.4245, 0.2527)  & (1.76e-11, 3.07e-11, 9.98e-05) &  (0.0009, 2, 456, 1) & 4581.41 \\ [3pt]
				(115.2327, 0.2722)  & (2.94e-11, 2.13e-11, 9.98e-05) &  (0.0009, 2, 456, 1) & 4592.07 \\ [3pt]
				(144.0409, 0.2916) & (1.73e-11, 2.46e-11, 9.99e-05) &  (0.0009, 2, 456, 1) & 4633.12 \\ [3pt]
				(172.8491, 0.3111) & (1.10e-11, 1.47e-11, 9.99e-05) &  (0.0009, 2, 456, 1) & 4809.25 \\ [3pt]
				\hline
		\end{tabular}}
	\end{center}
	\caption{Numerical performance of the network with $(n,m,k)=(1000, 15202, 457)$ for different $\rho_D$ and $\rho_S$: $R_p$, $R_d$ and {\tt rel\_gap} stand for the relative primal feasibility, dual feasibility and relative duality gap obtained by {\tt sGS-ADMM}, respectively; we use {\tt rel\_err}, {\tt r\_dim}, {\tt nz\_S}, {\tt mis\_O} to denote the relative error, recovery embedding dimension, detected outlier number and outlier detection error; {\tt cpu(s)} is the total computational time (in seconds) of {\tt sGS-ADMM}.}
	\label{tab:T1}
\end{table*}

\vskip 10 true pt

\noindent{\bf Example 2}. We use this example to illustrate the quality of the proposed estimators for  problems with different dimensions. Ten simulation networks with $n=\{200,400,\ldots,2000\}$ random points in the square area $[0, 100]\times [0, 100]$ are generated in a similar manner as Example 1. The sample sizes of the corresponding observation operators and the numbers of outliers for different networks are reported in Table \ref{tab:T2}. The black lines in both Figure \ref{fig:2-sub1} and \ref{fig:2-sub2} represent the  theoretical (relative) upper bounds in Theorem \ref{thm:error_bound}, i.e.,  $\max(\Gamma_1,\Gamma_2)/(1+\|\overline{D}\|+\|\overline{S}\|)$ and $\Gamma_1$ and $\Gamma_2$ are defined in \eqref{eq:def-Gamma1} and \eqref{eq:def-Gamma2}, respectively. The blue lines in both Figure \ref{fig:2-sub1} and \ref{fig:2-sub2} are the square sum of relative  errors with respect to $\widehat{D}$ and $\widehat{S}$, i.e., $\|\widehat{D}-\overline{D}\|^2+\|\widehat{S}-\overline{S}\|^2/(1+\|\overline{D}\|+\|\overline{S}\|)$. It can be seen clearly from Figure \ref{fig:2-sub1} and \ref{fig:2-sub2} that the square sum of (relative)  errors with respect to $\widehat{D}$ and $\widehat{S}$ is smaller than theoretical (relative) upper bounds defined by \eqref{eq:def-Gamma1} and \eqref{eq:def-Gamma2}. It can be seen from Figure \ref{fig:2-sub1} that the proposed convex model \eqref{pr:cvxopt}  provides the estimators $\widehat{D}$ with the true EDM dimension, i.e., ${\rm rank}(-J\widehat{D}J)=r=2$ in all ten networks. Similar with {\bf Example 1}, we note that for these ten cases, the $r$-th eigenvalues $\lambda_r(-J\overline{D}J)$ of $-J\overline{D}J$ are in the order of $O(10^5)$, and the upper bounds defined in \eqref{eq:def-Gamma1} and \eqref{eq:def-Gamma2} satisfy $\max\{\Gamma_1^{1/2},\Gamma_2^{1/2}\}\approx 10^{3}$. One the other hand, the red line in Figure \ref{fig:2-sub2} indicates the number of outlier detection errors obtained by \eqref{pr:cvxopt} for different networks. We know from Figure \ref{fig:2-sub2} and Table \ref{tab:T2} that for all networks with different dimension scales, the detected outlier numbers ${\tt nz\_S}\le k$ and ${\tt nz\_S}+{\tt mis\_O}=k$, which implies that for each case, the outlier detection errors if exist are the false-negative errors ({\tt FN}) and ${\rm supp}(\widehat{S})\subseteq{\rm supp}(\overline{S})$. Similarly with {\bf Example 1}, this result is consistent with the outlier detection guarantee results proposed in Theorem \ref{thm:rank-sparse-guarantee}. Also, it is worth noting that $10\approx\min\left\{\overline{S}_{ij} \mid (i,j)\in{\rm supp}(\overline{S})\right\} \ll \|\widehat{S}-\overline{S}\|\approx 10^3$ for these cases. The numerical details are reported in Table \ref{tab:T2}.

\begin{figure*}
	\centering
	\begin{subfigure}{.5\textwidth}
		\centering
		\includegraphics[width=1\linewidth]{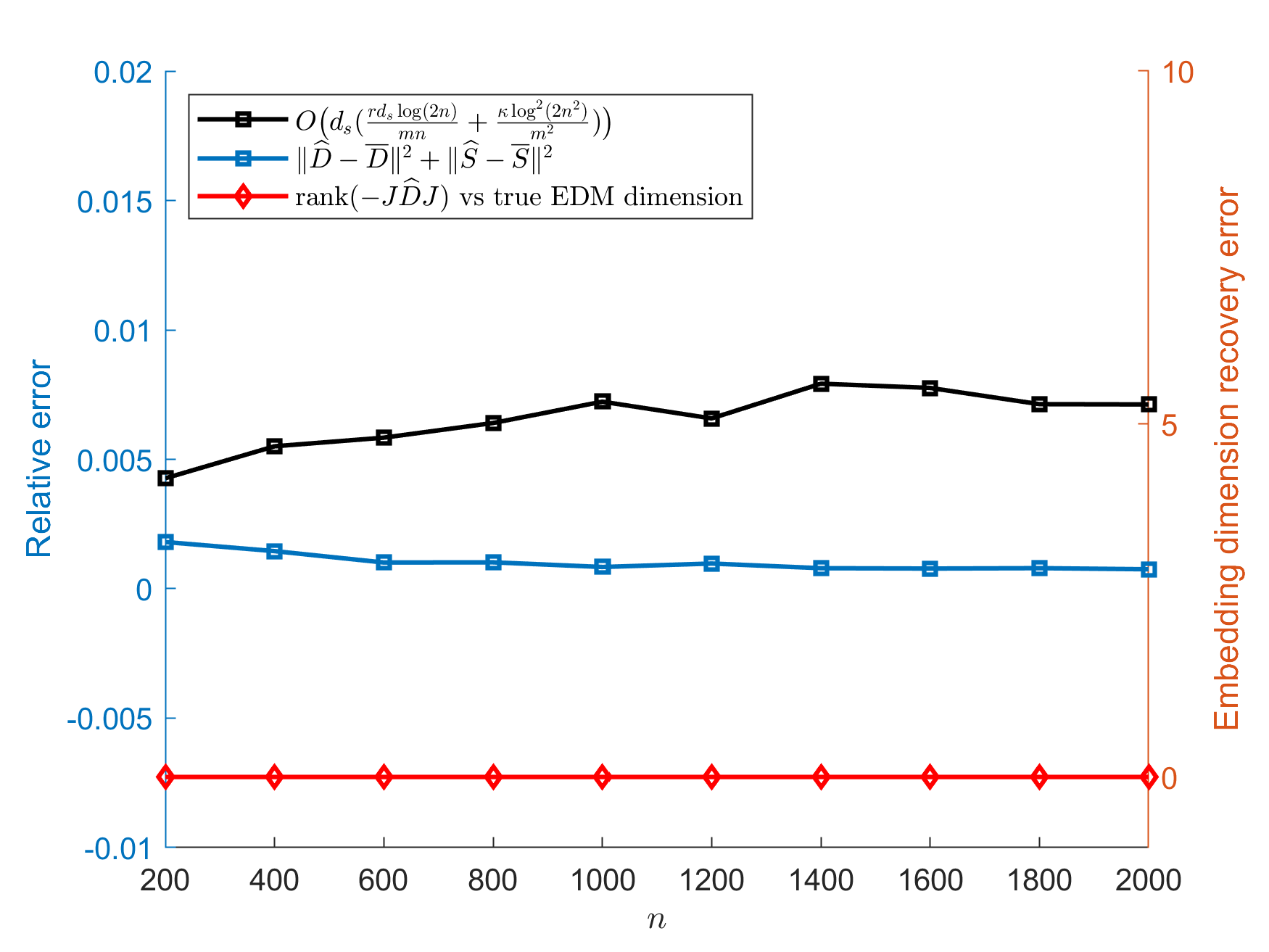}
		\caption{Relative error and dimension recovery error}
		\label{fig:2-sub1}
	\end{subfigure}%
	\begin{subfigure}{.5\textwidth}
		\centering
		\includegraphics[width=1\linewidth]{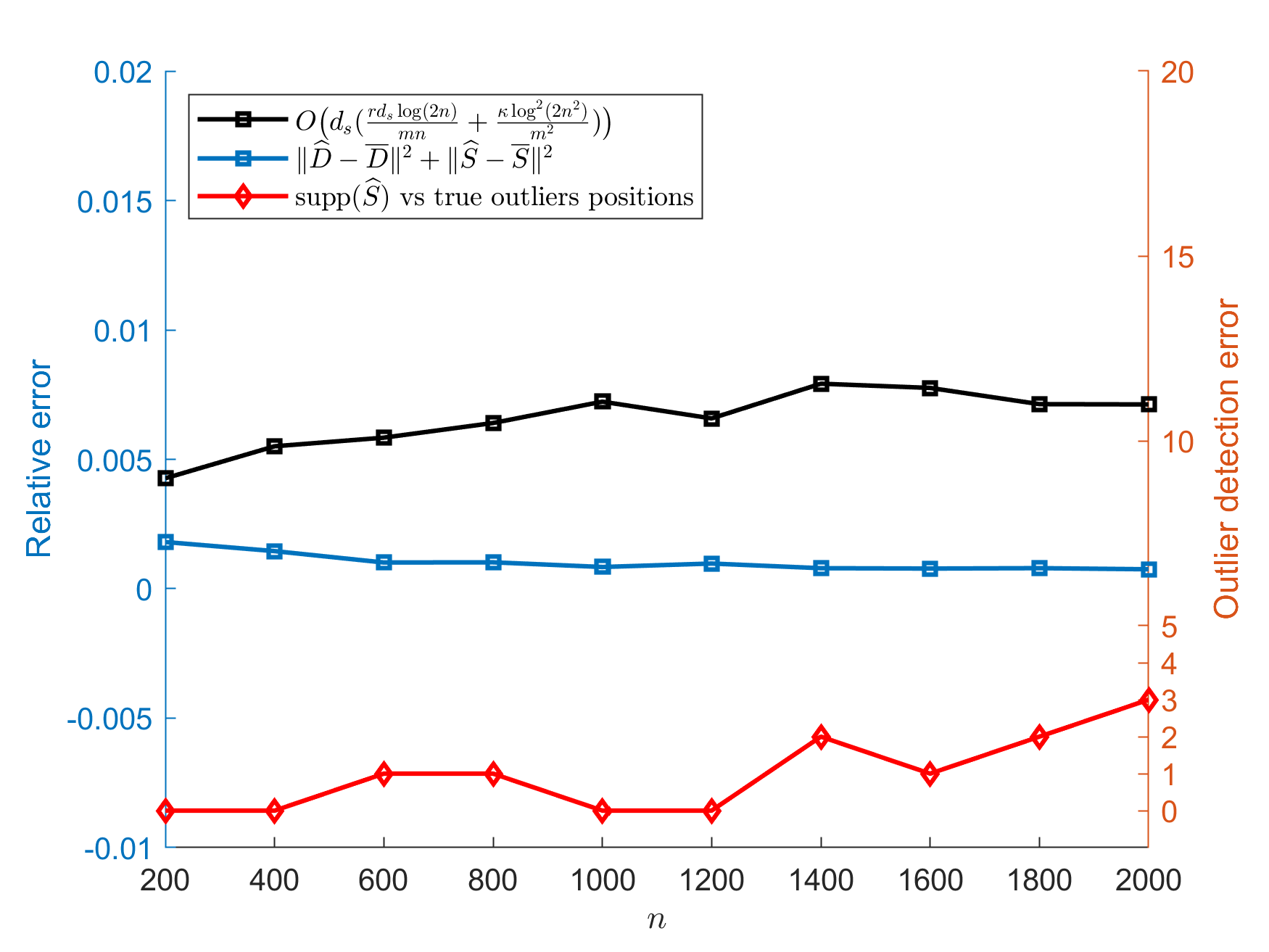}
		\caption{Relative error and outlier detection error}
		\label{fig:2-sub2}
	\end{subfigure}
	\caption{Performance comparison for networks with different dimension scales}
	\label{fig:dimension}
\end{figure*}

\begin{table*}
	\begin{center} 
		{\footnotesize\begin{tabular}{cc|cc|c}
				\hline
				$(n,m,k)$  & $(\rho_D,\rho_S)$  & ($R_p$,$R_d$,{\tt gap}) &  ({\tt rel\_err},{\tt r\_dim},{\tt nz\_S},{\tt mis\_O}) & {\tt cpu(s)}\\ [3pt]
				\hline
				(200, 2397, 72) & (393.7605, 1.5738)   & (8.44e-11, 7.14e-11, 9.95e-05) &  (0.0018, 2, 72, 0) & 17.30 \\ [3pt]
				(400, 5348, 161) & (157.6384, 0.6341)   & (2.67e-11, 7.02e-11, 9.97e-05) &  (0.0015 , 2, 161 0) &  193.20 \\ [3pt]
				(600, 8509, 256) & (157.3630, 0.6350)   & (3.04e-11, 4.36e-11, 1.00e-04) &  (0.0010, 2, 255, 1) & 688.16 \\ [3pt]
				(800, 11805, 355) & (111.6860, 0.4516)   & (4.87e-11, 3.19e-11, 9.99e-05) &  (0.0010, 2, 354, 1) & 1974.85 \\ [3pt]
				(1000, 15202, 457) & (125.0299, 0.5063)   & (1.92e-11, 2.01e-11, 1.00e-04) &  (0.0008, 2, 457, 0) & 4583.55 \\ [3pt]
				(1200, 18680, 561) & (69.6896, 0.2825)   & (2.70e-11, 2.87e-11, 9.99e-05) &  (0.0010, 2, 561, 0) & 7161.07 \\ [3pt]
				(1400, 22225, 667) & (88.5303, 0.3592)   & (1.96e-11, 1.71e-11, 9.99e-05) &  (0.0008, 2, 665, 2) & 12641.73 \\ [3pt]
				(1600, 25827, 775) & (78.9857, 0.3207)   & (1.49e-11, 1.57e-11, 1.00e-04) &  (0.0008, 2, 774, 1) & 21380.81 \\ [3pt]
				(1800, 29480, 885) & (52.1488, 0.2119)   & (1.69e-11, 2.11e-11, 1.00e-04) &  (0.0008, 2, 883, 2) & 33169.30 \\ [3pt]
				(2000, 33177, 996) & (44.9024, 0.1825)   & (2.24e-11, 1.93e-11, 9.99e-05) &  (0.0008, 2, 993, 3) & 48297.41 \\ [3pt]
				\hline
		\end{tabular}}
	\end{center}
	\caption{Numerical performance for networks with different dimension scales: $n,m,k$ indicate the number of points in network, sample size of the observation operator and number of outliers; $R_p$, $R_d$ and {\tt rel\_gap} stand for the relative primal feasibility, dual feasibility and relative duality gap obtained by {\tt sGS-ADMM}, respectively; we use {\tt rel\_err}, {\tt r\_dim}, {\tt nz\_S}, {\tt mis\_O} to denote the relative error, recovery embedding dimension, detected outlier number and outlier detection error; {\tt cpu(s)} is the total computational time (in seconds) of {\tt sGS-ADMM}. }
	\label{tab:T2}
\end{table*}

\vskip 10 true pt

\noindent{\bf Example 3}. Finally, we conduct an experiment to verify the proposed sufficient condition for the outlier detection in Theorem \ref{thm:rank-sparse-guarantee}. Consider a simulation network with $n=200$ points which are randomly located in the area $[0,100]\times[0,100]$. The observation operator ${\cal O}_{\Omega}$ and i.i.d. noise errors $\xi_{ij}$ are generated in the same manner as those in {\bf Example 1} \& {\bf 2}. Moreover, we randomly add $k=24$  outliers errors, which are the i.i.d. uniform random variables with different magnitudes such that $\overline{S}$ satisfies one of the following conditions, respectively: (a) $\min\left\{\overline{S}_{ij} \mid (i,j)\in{\rm supp}(\overline{S})\right\}\approx 1\times 10$ denoted by the ``small magnitude" of $\overline{S}$, (b) $\min\left\{\overline{S}_{ij} \mid (i,j)\in{\rm supp}(\overline{S})\right\}\approx 10^2$ denoted by the ``middle magnitude" of $\overline{S}$, and (c) $\min\left\{\overline{S}_{ij} \mid (i,j)\in{\rm supp}(\overline{S})\right\}\approx 5\times 10^3$ denoted by the ``large magnitude" of $\overline{S}$. Note that in this example, we have  $\max\{\Gamma_1^{1/2},\Gamma_2^{1/2}\}\approx 10^3$. By \eqref{eq:suff-c-outlier}, we know that the condition $\|\widehat{S}-\overline{S}\|<\min\left\{\overline{S}_{ij} \mid (i,j)\in{\rm supp}(\overline{S})\right\}$ is satisfied in the ``large magnitude" case. First,  the estimators $\widehat{D}$ obtained by the convex model \eqref{pr:cvxopt} in all cases satisfy ${\rm rank}(-J\widehat{D}J)=r=2$, since the $r$-th eigenvalues $\lambda_r(-J\overline{D}J)$ of $-J\overline{D}J$ are in the order of $O(10^5)$. For comparison, we report the relative errors and the outlier detection errors after $100$ Monte Carlo simulation runs in Figure \ref{fig:diff-S}. Note that for all cases, the outlier detection errors if exist are the false-negative errors ({\tt FN}) and ${\rm supp}(\widehat{S})\subseteq{\rm supp}(\overline{S})$, since the detected outlier numbers ${\tt nz\_S}\le k=$ and ${\tt nz\_S}+{\tt mis\_O}=k$. For the small and middle magnitude of $\overline{S}$ cases (Figure \ref{fig:3-sub1} and \ref{fig:3-sub2}), only a few estimators $\widehat{S}$ satisfy ${\rm supp}(\widehat{S})={\rm supp}(\overline{S})$ exactly ($13$ out of $100$ MC simulation runs for the small magnitude case; $77$ out of $100$ MC simulation runs for the middle magnitude case). However, it is worth noting that for the large magnitude of $\overline{S}$ case (Figure \ref{fig:3-sub3}),  with probability $1$, the estimator $\widehat{S}$ obtained by \eqref{pr:cvxopt} satisfies ${\rm supp}(\widehat{S})={\rm supp}(\overline{S})$ exactly ($100$ out of $100$ MC simulation runs). Also the empirical cumulative distribution function\footnote{Let $x_1,\ldots, x_n$ be independent, identically distributed real random variables. The corresponding empirical distribution function $F_n(t)$ is defined as $F_n(t)=\frac{1}{n}\sum_{i=1}^{n}\delta_{x_i\le t}$, where $\delta_{x_i\le t}$ is the indicator of event $x_i\le t$.} (CDF) of different cases are reported in Figure \ref{fig:3-sub4}.

\begin{figure*}
	\centering
	\begin{subfigure}{.5\textwidth}
		\centering
		\includegraphics[width=1\linewidth]{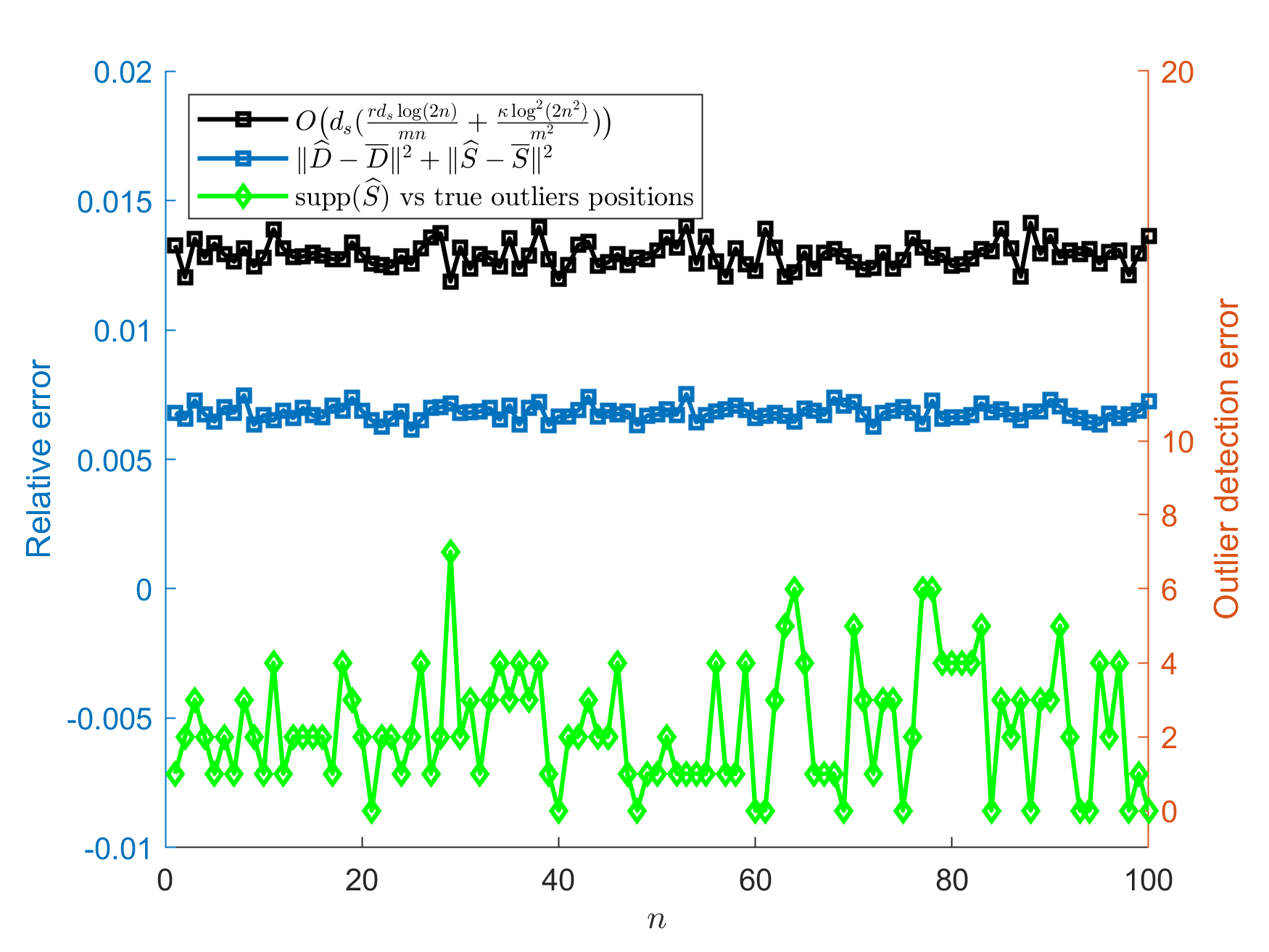}
		\caption{The small magnitude $\overline{S}$}
		\label{fig:3-sub1}
	\end{subfigure}%
	\begin{subfigure}{.5\textwidth}
		\centering
		\includegraphics[width=1\linewidth]{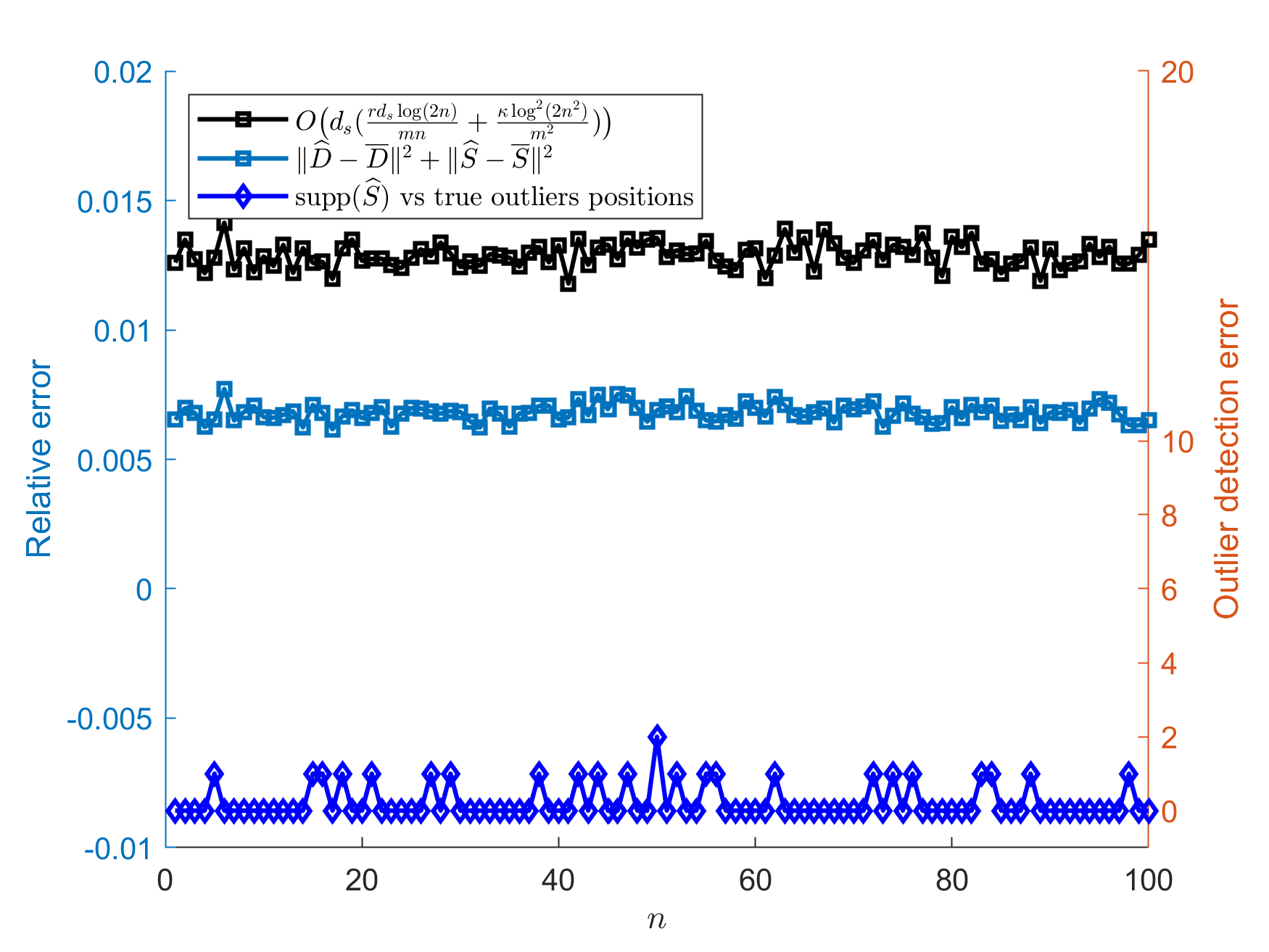}
		\caption{The middle magnitude $\overline{S}$}
		\label{fig:3-sub2}
	\end{subfigure}
	\begin{subfigure}{.5\textwidth}
		\centering
		\includegraphics[width=1\linewidth]{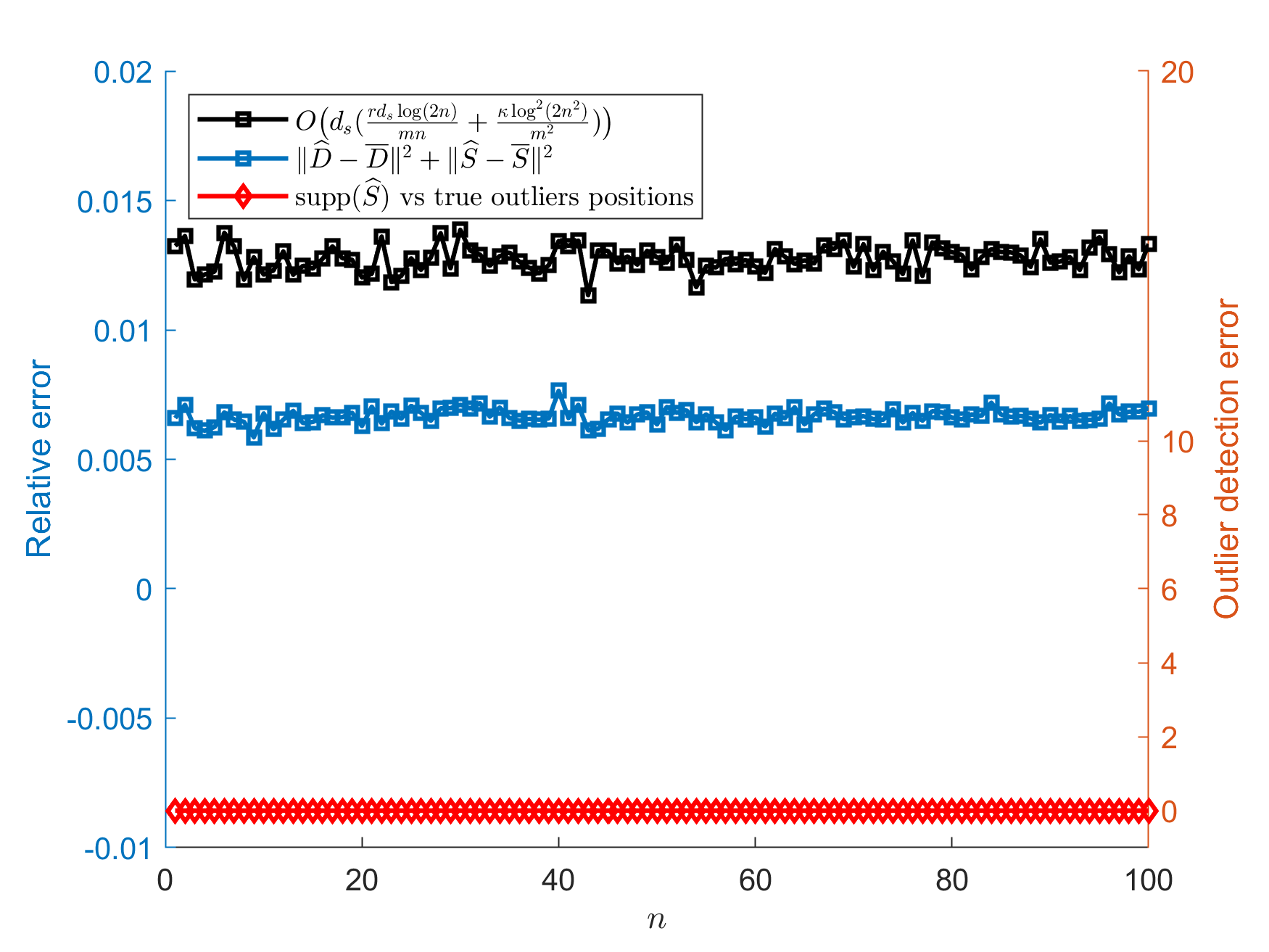}
		\caption{The large magnitude $\overline{S}$}
		\label{fig:3-sub3}
	\end{subfigure}%
	\begin{subfigure}{.5\textwidth}
		\centering
		\includegraphics[width=1\linewidth]{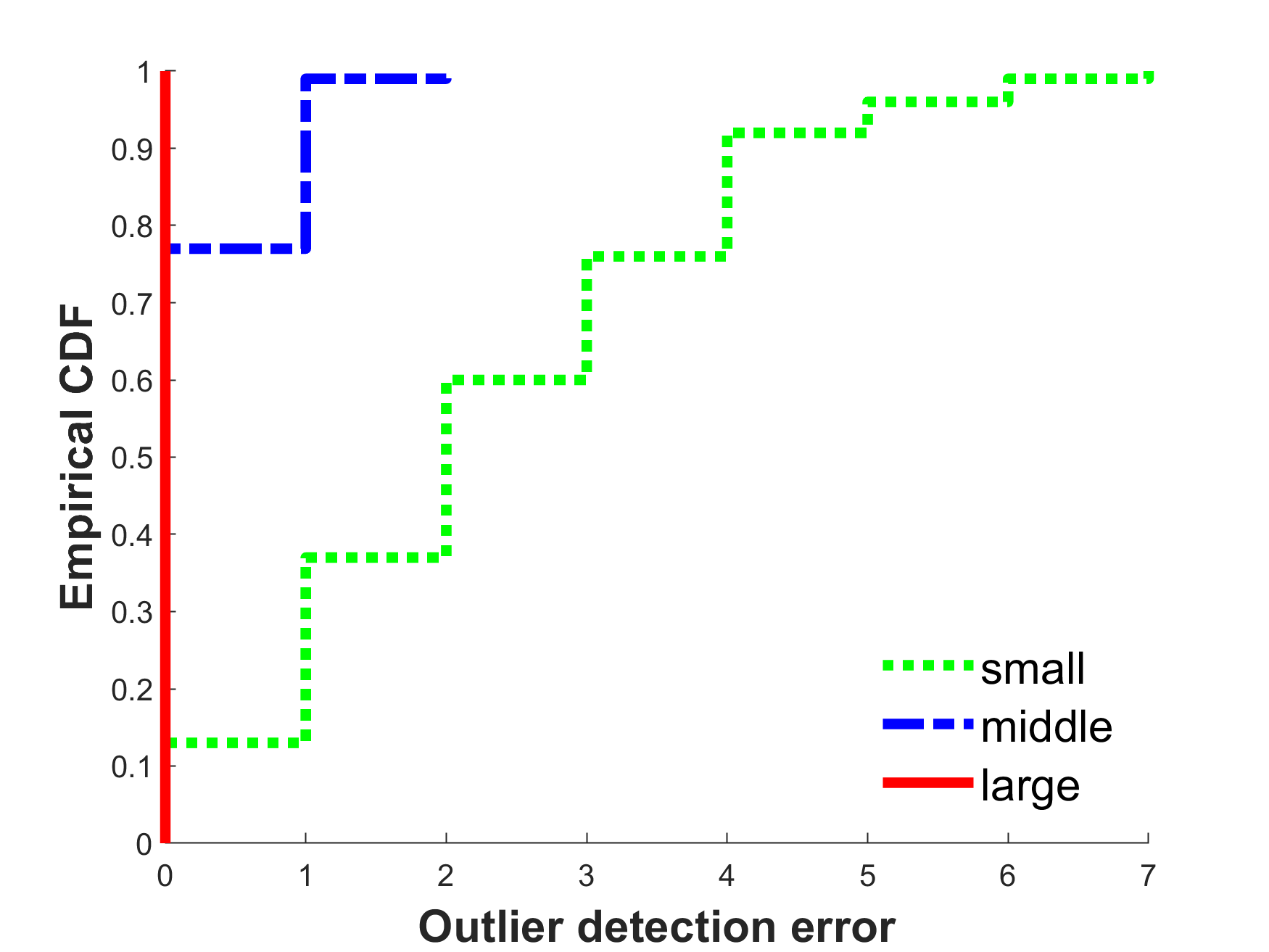}
		\caption{Empirical CDF}
		\label{fig:3-sub4}
	\end{subfigure}
	\caption{Relative error and outlier detection error for different magnitudes of $\overline{S}$ }
	\label{fig:diff-S}
\end{figure*}

\section{Conclusions}\label{section:conclusion}

Euclidean embedding from noisy observations containing outliers is an important and challenging problem in statistics and machine learning. Many existing methods would struggle with outliers due to a lack of detection ability, while the matrix optimization based embedding model introduced in \cite{DQi2017b} usually can produce reliable embeddings and identify the outliers jointly. This paper aimed to explain this mysterious situation by studying the estimation error bounds and outliers detection ability of the proposed model. In particular, we show that the estimators obtained by the proposed method satisfy a non-asymptotic risk bound, implying that the model provides a high accuracy estimator with high probability when the order of the sample size is roughly the degree of freedom up to a logarithmic factor. Moreover, we show that under some mild conditions, the proposed model also can identify the outliers without any prior information with high probability. As we mentioned in the Section \ref{section:error-bounds}, Chen, et al. \cite{CFMYan2020}, derived a near-optimal statistical guarantee of the convex nuclear norm plus $l_1$-norm penalized model for the classical Robust PCA by building up the connection between the convex estimations and an auxiliary nonconvex optimization algorithm. It seems their approach would lead to some interesting error bound results for our EDM embedding model. However, it seems difficult to extend their results directly to the convex models involving "hard-constraints", e.g., the noisy correlation matrix recovery and the EDM estimation studied in this paper.  We plan to investigate those issues in the future.

%

\bibliographystyle{spmpsci}      

\bibliography{myref}

\end{document}